\documentclass[english]{article}

\usepackage{enumitem}
\usepackage{hyperref}

\usepackage{geometry} 
\usepackage{setspace}
\usepackage[utf8]{inputenc} 
\usepackage[T1]{fontenc}    
\usepackage{hyperref}       
\usepackage{url}            
\usepackage{booktabs}       
\usepackage{amsfonts}       
\usepackage{nicefrac}       
\usepackage{microtype}      
\usepackage{bbold}
\usepackage{amsmath}
\usepackage{amsthm}
\usepackage{algorithm}	
\usepackage{color}
\usepackage{enumitem}
\usepackage{subcaption}
\usepackage{graphicx}
\usepackage{cite}

\usepackage[noend]{algpseudocode}
\DeclareMathOperator*{\argmax}{arg\,max}
\DeclareMathOperator*{\argmin}{arg\,min}

\newtheorem{theorem}{Theorem}
\newtheorem{lemma}{Lemma}
\newtheorem{corollary}{Corollary}
\newtheorem{remark}{Remark}

\newcommand{\A}{\mathcal{A}}
\newcommand{\htheta}{\hat{\theta}}
\newcommand{\RR}{\mathbb{R}}
\newcommand{\la}{\langle}
\newcommand{\ra}{\rangle}
\newcommand{\T}{\mathcal{T}}
\newcommand{\R}{R}
\newcommand{\m}{M}
\newcommand{\vu}{U}

\usepackage{authblk}

\begin{document}

\title{Stochastic Linear Bandits Robust to Adversarial Attacks}

\author[1]{Ilija Bogunovic}
\author[2]{Arpan Losalka}
\author[1]{Andreas Krause}
\author[2]{Jonathan Scarlett}

\affil[1]{ETH Z\"urich}
\affil[2]{National University of Singapore}

\date{}

\sloppy
\onehalfspacing

\maketitle


\begin{abstract}
    We consider a stochastic linear bandit problem in which the rewards are not only subject to random noise, but also adversarial attacks subject to a suitable budget $C$ (i.e., an upper bound on the sum of corruption magnitudes across the time horizon).  We provide two variants of a Robust Phased Elimination algorithm, one that knows $C$ and one that does not.  Both variants are shown to attain near-optimal regret in the non-corrupted case $C = 0$, while incurring additional additive terms respectively having a linear and quadratic dependency on $C$ in general.  We present algorithm-independent lower bounds showing that these additive terms are near-optimal. In addition, in a contextual setting, we revisit a setup of diverse contexts, and show that a simple greedy algorithm is provably robust with a near-optimal additive regret term, despite performing no explicit exploration and not knowing $C$.    
\end{abstract}

\vspace*{-3ex}
\section{Introduction}
\vspace*{-1ex}

Over the past years, bandit algorithms have found application in computational advertising, recommender systems, clinical trials, and many more. These algorithms make online decisions by balancing between exploiting previously high-reward actions vs.~exploring less known ones that could potentially lead to higher rewards. Bandit problems can roughly be categorized \cite{lattimore2018bandit} into {\em stochastic bandits}, in which subsequently played actions yield independent rewards, and {\em adversarial bandits}, where the rewards are chosen by an adversary, possibly subject to constraints.  
%
A recent line of works has sought to reap the benefits of both approaches by studying bandit problems that are stochastic in nature, but with rewards subject to a limited amount of {\em adversarial corruption}.  Various works have developed provably robust algorithms~\cite{gupta2019better,lykouris2018stochastic,bogunovic2020corruption,li2019stochastic}, and attacks have been designed that cause standard algorithms to fail \cite{garcelon_adversarial_2020,gupta2019better,jun_adversarial_nodate,liu2019data}.

While near-optimal theoretical guarantees have been established in the case of independent arms \cite{gupta2019better}, more general settings remain relatively poorly understood or even entirely unexplored; see Section \ref{sec:related} for details.  Our primary goal is to bridge these gaps via a detailed study of stochastic {\em linear} bandits with adversarial corruptions.  In the case of a fixed finite (but possibly very large) set of arms, we develop an elimination-based robust algorithm and provide regret bounds with a near-optimal joint dependence on the time horizon and the adversarial attack budget, demonstrating distinct behavior depending on whether the attack budget is known or unknown.  In addition, we introduce a novel \emph{contextual} linear bandit setting under adversarial corruptions, and show that under a context diversity assumption, a simple greedy algorithm attains near-optimal regret under adversarial corruptions, despite having no built-in mechanism that explicitly encourages exploration or robustness. 

\vspace*{-1ex}
\subsection{Problem Setting} \label{sec:problem_setup}
\vspace*{-1ex}

We consider the stochastic linear bandit setting with a given set of arms $\A_0 \subset \RR^d$ of finite size $k$, and adversarially corrupted rewards. At each round $t \in \lbrace 1, \dots, T \rbrace$:
\begin{itemize}[leftmargin=5ex,itemsep=0ex,topsep=0.25ex]
  \item The learner chooses an action $A_t \in \A_0$.
  \item The adversary observes $A_t$ and decides upon the attack/corruption $c_t(A_t)$; in addition, $c_t(\cdot)$ may (implicitly) depend on other problem parameters, as detailed below.
  \item The learner receives a corrupted reward $Y_t$:
    \begin{equation}
      Y_t =  \langle \theta, A_t \rangle + \epsilon_t + c_t(A_t),
    \end{equation}
    where $\theta \in \RR^d$ is an unknown parameter vector, and $(\epsilon_t)_{t=1}^T$ is a random noise term, which is assumed to be zero-mean and $1$-sub-Gaussian.
\end{itemize}

We assume that the action feature vectors are unique, span $\RR^d$, and are bounded, i.e., $\| a\|_2 \leq 1, \forall a \in \A_0$.  We similarly make the standard assumption $\|\theta \|_2 \leq 1$, which implies that $|\langle \theta, a \rangle| \leq 1, \forall a \in \A_0$.

We consider an adversary/attacker that has complete knowledge of the problem -- it knows both $\A_0$ and $\theta$, and observes both the precise arm pulled and the noise realization $\epsilon_t$ before choosing its attack. The total {\em attack budget} of the adversary is given by $\sum_{t=1}^T |c_t(A_t)| \leq C$. We will consider both the cases that $C$ is known and unknown to the learner.

The goal of the learner is to minimize the {\em cumulative regret}, defined as
\begin{equation}
   R_T = \sum_{t=1}^T \max_{a \in \A_0} \langle \theta, a - A_t \rangle. \label{eq:regret_def}
\end{equation}
Broadly speaking, we say that an algorithm that attains low regret (e.g., sublinear scaling $R_T = o(T)$) is {\em corruption-tolerant} or {\em robust to adversarial attacks}.

As noted in \cite{lykouris2018stochastic}, one could alternatively count the corruption as being part of the reward and define regret with the corruption included.  Both notions are of interest depending on the application (e.g., depending on whether a fake ad click is considered beneficial or not).  The two notions differ by at most $O(C)$, whereas our upper bounds will contain at least an $O(C \log T)$ term.  In addition, in the multi-armed bandit setting, $\Omega(C)$ lower bounds were shown for both notions in \cite{lykouris2018stochastic}.

\vspace*{-1ex}
\subsection{Related Work} \label{sec:related}
\vspace*{-1ex}

Recent surveys on bandit algorithms can be found in \cite{lattimore2018bandit,slivkins2019introduction}; here we focus on the most relevant works considering \emph{stochastic} settings with adversarial corruptions and bandit attacks. 

Adversarial attacks on standard bandit algorithms (e.g., UCB, $\epsilon$-greedy, and Thompson sampling) were introduced for the case of independent arms (i.e., a classical multi-armed bandit setting) in \cite{jun_adversarial_nodate,liu2019data,liu2020action}, and for linear bandits in \cite{garcelon_adversarial_2020}.  We will use the latter in our experiments to test robustness of the proposed algorithms, along with other heuristic attacks.

In the case of \emph{independent} arms, Lykouris {\em et al.}~\cite{lykouris2018stochastic} show that a simple elimination algorithm with enlarged confidence bounds is robust and near-optimal when the attack budget $C$ is known.  For unknown $C$, randomized algorithm is given whose regret bound roughly amounts to scaling the uncorrupted regret 
by $C$, i.e., multiplicative dependence.  Subsequently, Gupta {\em et al.}~\cite{gupta2019better} gave an improved algorithm whose regret is near-optimal, with an {\em additive} dependence on $C$.


Bogunovic {\em et al.}~\cite{bogunovic2020corruption} consider corruption-tolerant bandits for functions with a bounded RKHS norm, which includes linear bandits as a special case.  The algorithm of \cite{bogunovic2020corruption} is based on that of \cite{lykouris2018stochastic}, and has analogous guarantees.  However, even in the case of known $C$, the best dependence obtained is multiplicative; the possibility of additive dependence was left as an open problem, which we resolve in this work in the linear case. 

Li {\em et al.}~\cite{li2019stochastic} also study 
stochastic linear bandits with adversarial corruptions.  A distinction in \cite{li2019stochastic} is that the regret bounds are {\em instance-dependent}, relying on positive gaps between the function values at corner points of the polyhedral domain.  These results are distinct from the {\em instance-independent} bounds with {\em a finite number of arms} that we seek in this paper, and neither can be deduced from the other; see \cite[App.~K]{bogunovic2020corruption} for further discussion, as well as Remark \ref{rem:comparison} below.  

It is worth noting that the above-mentioned works \cite{lykouris2018stochastic,gupta2019better,bogunovic2020corruption,li2019stochastic} consider a weaker adversary that cannot observe the current action, while this has often also been assumed when designing efficient bandit attacks \cite{jun_adversarial_nodate,liu2019data}. 
Our more powerful adversary has also been considered before (e.g., see \cite[Fig.~2]{liu2019data}), and naturally, any given upper bound on regret is stronger the more powerful of an adversary it applies to. 

In Appendix~\ref{sec:more_rw}, we discuss further existing works that are less directly related to ours compared  to those above, including distinct adversarial settings (e.g., handled by the EXP2 and EXP3 algorithms), ``best of both worlds'' results for stochastic and adversarial bandits, model mismatch and misspecification, and fractional/Huber-like contamination models. 

\begin{remark} \label{rem:comparison}
    Returning to the results in \cite{li2019stochastic}, one may note that instance-dependent bounds can potentially be transferred to instance-independent bounds.  However, we show in Appendix \ref{app:transfer} that doing this for the results in \cite{li2019stochastic} would at best lead to $R_T = O( T^{2/3} + \sqrt{CT} )$, which is strictly higher than than our analogous result (Theorem \ref{thm:regret_bound2}) whenever $C = o(T^{1/3})$.  This is despite the fact that we are considering a stronger adversary.  However, it should be kept in mind that the domains adopted are different (polyhedral vs.~finite), posing another hurdle that would need to be overcome to transfer results from one setting to the other.
\end{remark}




\vspace*{-1ex}
\subsection{Contributions}
\vspace*{-1ex}

Our main contributions are as follows:
\begin{itemize}[leftmargin=5ex,itemsep=0ex,topsep=0.25ex]
    \item For known $C$, we present a Robust Phased Elimination algorithm, and show that it recovers a near-optimal regret bound when $C = 0$, while incurring an additive $O(d^{3/2}C \log T)$ term (up to $\log \log (dT)$ factors) more generally.  A standard lower bound argument \cite{lykouris2018stochastic} shows that $\Omega(C)$ dependence is unavoidable, thus certifying the upper bound as being optimal up to logarithmic factors   when $d = O(1)$ (the precise $d$ dependence is not a main focus of our work).
    \item For unknown $C$, we modify our algorithm to gradually decrease its confidence bound enlargement term over time, and show that we only pay a further $O(C^2)$ term compared to the known $C$ case.  While this limits the regime of sublinear regret to $C = o(\sqrt T)$ (in contrast with $C = o(T)$ when $C$ is known), we additionally provide a novel algorithm-independent lower bound showing that this is unavoidable for any algorithm that achieves a near-optimal non-corrupted ($C = 0$) bound.  Thus, we prove a fundamental difficulty in being robust against our strong adversary when $C$ is unknown, and demonstrate a fundamental gap between the known $C$ and unknown $C$ settings.
    \item We introduce a linear contextual problem with adversarial attacks, and show that under the model of diverse contexts from \cite{kannan_smoothed_2018}, the greedy algorithm not only attains near-optimal regret in the uncorrupted setting (as shown in \cite{kannan_smoothed_2018}), but is also {\em robust to adversarial attacks}.
\end{itemize} 

\vspace*{-1ex}
\section{Algorithm and Regret Bounds} \label{sec:algorithm_and_regret_bound}
\vspace*{-1ex}

We present our Robust Phased Elimination algorithm in Algorithm~\ref{alg:cpe}, which builds on non-robust elimination algorithms \cite{lattimore2018bandit,lattimore2019learning,valko2014spectral}, with some important differences outlined in Remark \ref{rem:differences} below. The known $C$ vs.~unknown $C$ variants only differ on Line \ref{line:chooseC}.
The algorithm runs in epochs of exponentially increasing length and maintains a set of potentially optimal actions.
In every epoch, the following steps are performed: (i) compute a near-optimal experimental design over a set of potentially optimal actions, and play each action from this subset in proportion to the computed design (Lines 2-4); 
(ii) compute an estimate of $\theta$, and use it to eliminate actions that appear suboptimal (Lines 5-6). We proceed by describing these steps in more detail.

\textbf{Action selection.} To introduce the action selection procedure, consider the problem of finding a probability distribution $\zeta: \A \rightarrow [0,1]$ that solves the following:
\begin{equation} \label{eq:experimental_design}
   {\rm minimize}_{\zeta} ~~~ {\max}_{a \in \A} \| a \|_{\Gamma(\zeta)^{-1}}^2 \quad \text{s.t. } ~~\sum_{a \in \A} \zeta(a) = 1,
\end{equation}
where $\Gamma(\zeta) = \sum_{a \in \A} \zeta(a) aa^T$, and $\| a\|_{M} = \sqrt{a^T M a}$.
A classical result from \cite{kiefer1960equivalence} states that the optimal solution $\zeta^*$ exists, and achieves $\max_{a \in \A} \| a \|_{\Gamma(\zeta^*)^{-1}}^2 = d$ with $|\text{supp}(\zeta^*)| \le \tfrac{d(d+1)}{2}$. 
%
%
For our purposes, however, it suffices to solve the problem in~\eqref{eq:experimental_design} only near-optimally. As noted in \cite{lattimore2019learning}, there exists a near-optimal design of smaller support than $d(d+1)/2$. In particular, if $\A$ spans $\RR^d$,\footnote{See Remark \ref{rem:observations} below for the general case.} then we can \emph{efficiently} compute $\zeta: \A \rightarrow [0,1]$ such that 
\begin{equation} \label{eq:solve_near_optimal_design}
  \max_{a \in \A} \| a \|_{\Gamma(\zeta)^{-1}}^2 \leq 2d, ~~ |\text{supp}(\zeta)| \leq 4d(\log \log d + 18)
\end{equation}
This follows from~\cite[Proposition 3.17]{booktodd}, who provide a polynomial-time Frank-Wolfe algorithm. 

Hence, in every epoch $h$, the algorithm recomputes a near-optimal design from \eqref{eq:solve_near_optimal_design} over a subset of the actions that are still potentially optimal, i.e., $\A_{h}$. It then plays each action from this subset in proportion to the computed design, but it also makes sure that every arm in its support is played at least some minimal number of times $\lceil \nu m_h \rceil$, where $\nu$ is an input truncation parameter to be chosen below, and $m_h$ is an exponentially increasing parameter with respect to the epoch length. 


\textbf{Parameter estimation and arm elimination.} Consider the estimator given in \eqref{eq:estimators2}.
This estimator only depends on the observations received in the current epoch, and hence, it is not affected by attacks suffered during previous epochs. However, it can still be biased due to the adversarial attacks suffered in the current epoch, and we need to account for this bias. In Lemma \ref{lem:robust_conf} (Appendix \ref{sec:pf_pe}), for any of the remaining potentially optimal actions, we bound the difference of the true mean reward and estimated one, and show that this error grows linearly with the total attack budget $C$.    
Hence, the algorithm makes use of the enlarged confidence bounds in \eqref{eq:retain_arms_condition} to retain potentially optimal arms.
Moreover, we show that when $C$ is known, our estimator is guaranteed to have sufficient accuracy so that the optimal arm is always retained in \eqref{eq:retain_arms_condition} with high probability.  For unknown $C$, this is not always the case, but we can control the level of suboptimality of the arms that are retained.

The estimator of $\theta$ is robust due to the fact that it averages the rewards corresponding to the same played action, reducing the effect of the attack. Intuitively, actions that have higher importance according to the found near-optimal design are played more times than others. Consequently, it is harder for the adversary to corrupt them as it needs to use more of the attack budget.  In addition, due to the introduced truncation, the algorithm plays each arm in the support of the computed design a fixed minimum number of times. 

\begin{algorithm}[t!]
    \caption{Robust Phased Elimination}
    \label{alg:cpe}
    \begin{algorithmic}[1]
        \Require Actions $\A_0 \subset \RR^d$, confidence $\delta \in (0,1)$, truncation parameter $\nu \in (0,1)$, time horizon $T$
        \State Initialize\footnotemark~$m_0 = 4d(\log \log d + 18)$, and for each $h \in \{0,1,\dotsc,\log_2 T-1\}$, set $\hat{C}_h = C$ for known $C$, or $\hat{C}_h = \min\lbrace \frac{\sqrt{T}}{m_0 \log_2 T}, m_0\sqrt{d} 2^{\log_2 T-h} \rbrace$ for unknown $C$.  Initialize $h=0$. \label{line:chooseC}
        \State Compute design $\zeta_h: \A_{h} \rightarrow [0,1]$ such that
        \begin{equation} \label{eq:solve_near_optimal_design2}
            \max_{a \in \A_{h}} \| a \|_{\Gamma(\zeta_h)^{-1}}^2 \leq 2d, 
            ~ \text{and} ~ |\text{supp}(\zeta_h)| \leq m_0. 
        \end{equation}
        \State Set $u_h(a)=0$ if $\zeta_h(a) = 0$, and $u_h(a)=\lceil m_{h} \max\lbrace\zeta_h(a), \nu\rbrace \rceil$ otherwise.
        \State Take each action $a \in \A_{h}$ exactly $u_h(a)$ times with corresponding features $(A_t)_{t=1}^{u_h}$ and rewards $(Y_t)_{t=1}^{u_h}$ (implicitly depending on $h$), where $u_h = \sum_{a \in \A_{h}} u_h(a)$.
        \State Estimate the parameter vector $\hat{\theta}_h$:
        \begin{gather} \label{eq:estimators2}
          \hat{\theta}_h = \Gamma_h^{-1} \sum_{t=1}^{u_h} A_t u_h(A_t)^{-1} \sum_{s \in \T(A_t)}  Y_s,  \\ \Gamma_h = \sum_{a \in \A_{h}} u_h(a) aa^T, 
        \end{gather}
        where $\T(a) = \big\lbrace  s \in \lbrace 1,\dots,u_h \rbrace : A_s = a \big\rbrace$ is the set of times at which arm $a$ is played.
        \State Update the active set of arms:                                     
        \begin{align} 
          &\A_{h+1} \leftarrow \Big\lbrace a \in \A_{h}: \max_{a' \in \A_{h}} \langle \hat{\theta}_h, a'- a \rangle \nonumber \\ &~~\leq 2 \sqrt{\tfrac{4d}{m_{h}}\log\big(\tfrac{1}{\delta}\big)} +  
  \tfrac{2\hat{C}_h}{m_h \nu} \sqrt{4d(1 + \nu m_0)} \Big\rbrace. \label{eq:retain_arms_condition}
        \end{align}
        \State Set $m_{h+1} \leftarrow 2m_{h}$, $h \leftarrow h+1$ and return to step 2 (terminating after $T$ total arm pulls).
    \end{algorithmic}
\end{algorithm}



\begin{remark} \label{rem:observations}
    The following observations from \cite{lattimore2019learning} are useful: (i) While \eqref{eq:solve_near_optimal_design} is stated assuming the arms span $\RR^d$, we can simply work in the lower-dimensional subspace otherwise (e.g., when $k < d$); (ii) We can extend the algorithm and its analysis to infinite-arm settings using a covering argument.
\end{remark}

\begin{remark} \label{rem:differences}
    Phased elimination algorithms (without robustness to adversarial attacks) have previously been considered in various settings, including the standard setting \cite[Ch.~22]{lattimore2018bandit}, misspecified setting \cite{lattimore2019learning}, and graph bandits \cite{valko2014spectral}. Among these, our algorithm is most similar to~\cite{lattimore2019learning}, but has several important differences: (i) We use a different and more robust estimator of $\theta$; (ii) The confidence bounds are enlarged and explicitly depend on $C$ to account for adversarial corruptions; (iii) The truncation parameter is introduced to ensure that each arm is pulled enough; (iv) In the unknown $C$ case, we need to carefully choose the sequence $\hat{C}_h$ to trade off robustness against aggressiveness in eliminating suboptimal arms; (v) In contrast to the vast majority of existing elimination algorithms, the optimal arm {\em may} be eliminated in the unknown $C$ setting (i.e., the confidence bounds may not be ``valid''), but this only occurs when the best remaining arm is still good enough to control the regret.
\end{remark}

\vspace*{-1ex}
\subsection{Upper Bounds on Regret}\label{sec:upper_bounds}
\vspace*{-1ex}


We first provide a regret bound for the known $C$ case, proved in Appendix \ref{sec:pf_pe}.

\begin{theorem} \label{thm:regret_bound}
    For any attack budget $C \ge 0$, with probability at least $1 - \delta$, the Robust Phased Elimination algorithm with known $C$ and truncation parameter $\nu = \tfrac{1}{4d(\log \log d + 18)}$ satisfies
      \begin{equation}
        R_T = \tilde{O}\Big(\sqrt{dT\log\big(\tfrac{k}{\delta}\big)} + Cd^{3/2}\log T\Big), \label{eq:RT1}
      \end{equation}
    where the notation $\tilde{O}(\cdot)$ hides $\log\log(dT)$ factors.
\end{theorem}

When $C = 0$, we recover the scaling of \cite[Thm.~22.1]{lattimore2018bandit}, which is near-optimal in light of known lower bounds \cite{dani2008stochastic}.  In Section~\ref{sec:lower_bounds}, we will argue that the second term is also near-optimal. 

\footnotetext{When $d = 1$, we have $\log \log d = -\infty$, but the results hold with $\log \log d$ replaced by $\log(1+\log d)$.} 


Next, we consider the case that the total attack budget $C$ is unknown to the learner.  We start by discussing the choice of $\hat{C}_h$ in Algorithm \ref{alg:cpe}.  Let $H$ be the number of epochs, and note that $\tilde{H} = \log_2 T$ be a deterministic upper bound on $H$ (see Appendix \ref{sec:regret_analysis} for a short proof).
%
Then, the choice in Algorithm \ref{alg:cpe} can be rewritten as $\hat{C}_h =  \min \lbrace \frac{\sqrt{T}}{m_0 \log_2 T}, m_0\sqrt{d} 2^{\tilde{H}-h} \rbrace$.  Observe that the epochs' lengths $u_h$ and corruption thresholds $\hat{C}_h$ are exponentially increasing and decreasing, respectively. It follows that the algorithm is more cautious in early epochs (i.e., uses larger thresholds).  Our second main result stated is as follows, and proved in Appendix \ref{sec:pf_pe}.

\begin{theorem} \label{thm:regret_bound2}
    For any 
    $C \le \frac{\sqrt T}{4d(\log \log d + 18) \log T}$, with probability at least $1 - \delta$, the Robust Phased Elimination algorithm with unknown $C$ and truncation parameter $\nu = \tfrac{1}{4d(\log \log d + 18)}$ satisfies   
      \begin{equation}
        R_T = \tilde{O}\Big(\sqrt{dT\log\big(\tfrac{k}{\delta}\big)} + Cd^{3/2}\log T + C^2\Big). \label{eq:RT2}
      \end{equation}
\end{theorem}

This result matches that of Theorem \ref{thm:regret_bound}, but with an additional penalty of $C^2$.  In fact, due to this penalty, the regret bound \eqref{eq:RT2} trivially holds when $C = \Omega(\sqrt T)$, because we have $R_T \le 2T$ due to our assumption of bounded rewards.  If $d = \omega(1)$, then there still remains the regime where $\frac{\sqrt T}{(d \log\log d)  \log T} \ll C \ll \sqrt{T}$, but in any case, one can slightly increase the final term and state that $R_T = \tilde{O}\big(\sqrt{dT\log\big(\tfrac{k}{\delta}\big)} + Cd^{3/2}\log T + C^2d^2  (\log T)^2\big)$ for arbitrary $C$.  

At this stage, observing that our regret bound is not sublinear in $T$ when $C = \Omega(\sqrt{T})$, the natural question arises as to whether attaining such a goal is impossible for all robust bandit algorithms.
In the following subsection, we use an algorithm-independent lower bound to provide a partial answer to this question; specifically, such a goal is indeed impossible (up to logarithmic factors) whenever the algorithm is required to have order-optimal regret in the uncorrupted ($C=0$) case.



\vspace*{-1ex}
\subsection{Algorithm-Independent Lower Bounds on Regret}\label{sec:lower_bounds}
\vspace*{-1ex}

Using the same reasoning as the standard multi-armed bandit setting \cite{lykouris2018stochastic}, it is straightforward to see that $\Omega(C)$ regret is unavoidable: The adversary can simply shift all rewards to zero for the first $C$ rounds, and the learner cannot do better than random guessing.  For completeness, this argument is given in more detail in Appendix \ref{sec:pf_lower}.  This argument holds even when $C$ is known, and thus, we see that the second term in Theorem \ref{thm:regret_bound} is optimal up to at most an $\tilde{O}(\log T)$ factor for fixed $d$.  
We expect that an improvement on the $d^{3/2}$ dependence may be possible, but the following result, proved in Appendix \ref{sec:pf_lower}, shows that at least $\Omega(Cd)$ is unavoidable.

\begin{theorem} \label{thm:lower_d}
    For any dimension $d$, there exists an instance with $k = d$ such that any algorithm (even with knowledge of $C$) must incur $\Omega(Cd)$ regret with probability at least $\frac{1}{2}$.
\end{theorem}

Next, we provide another lower bound that allows us to claim that the regret bound of Theorem~\ref{thm:regret_bound2} is near-optimal in the case that $C$ is unknown and the uncorrupted regret is required to be near-optimal.

\begin{theorem} \label{thm:lower_C}
    For $d=2$ and $k=2$, for any algorithm that guarantees $R_T \le \bar{R}^{(0)}_T$ with probability at least $1 - \delta$ for a given uncorrupted regret bound $\bar{R}^{(0)}_T \le \frac{T}{16}$ when $C = 0$, there exists an instance in which $R_T = \Omega(T)$ with probability at least $1-\delta$ when $C = 2 \bar{R}^{(0)}_T$. 
\end{theorem}

The proof is given in Appendix \ref{sec:pf_lower}.  While we focus on the simplest case $d = k = 2$, the proof can also be adapted to more general choices.  

{\bf Discussion.} Consider the general goal of attaining a regret upper bound of the form
\begin{equation}
    R_T \le \bar{R}^{(0)}_T + f(C) \log T,
\end{equation}
for some $f(\cdot)$ satisfying $f(0) = 0$.  Here we let the second term contain a $\log T$ factor in accordance with our upper bounds, but the following discussion still applies with only minor modifications when the $\log T$ factor is changed to ${\rm poly}(\log T)$ or similar.

At first glance, it appears that $f(C)$ should ideally be linear in $C$, and $\bar{R}^{(0)}_T$ should ideally be an order-optimal regret bound for the non-corrupted setting. 
However, Theorem \ref{thm:lower_C} shows that we cannot have both terms exhibiting their ``ideal'' behavior simultaneously.  To see this, note that the ideal uncorrupted regret bound behaves as $\bar{R}^{(0)}_T = \tilde{\Theta}(\sqrt{T})$ (for fixed $d$, $k$, and $\delta$) \cite{dani2008stochastic,lattimore2018bandit}.  Then, to be consistent with Theorem \ref{thm:lower_C}, we require $f(C) \log T = \tilde{\Omega}(T)$ for $C = \Theta(\sqrt{T})$, and hence $f(C) = \tilde{\Omega}\big( \frac{C^2}{\log C} \big)$.

On the other hand, it may be possible remove the $C^2$ term from $f(C)$ (i.e., improve robustness), and to attain sublinear regret for certain cases with $C = \Omega(\sqrt{T})$, if one is willing to pay the price of a worse uncorrupted regret bound. This idea is left for future work.

\subsection{Summary of Upper vs.~Lower Bounds}

We conclude this section with a short summary of how the upper and lower bounds compare in various scaling regimes of $C$ and $T$, when the other parameters $(d,k,\delta)$ are held fixed:
\begin{itemize}[leftmargin=5ex,itemsep=0ex,topsep=0.25ex]
    \item When $C$ is known, the optimal regret is between $\Omega(\sqrt{dT} + C\big)$ and $\tilde{O}(\sqrt{dT} + C \log T\big)$ for any $C \le T$;
    \item For $C = O\big( \frac{T^{1/4}}{\log T} \big)$, the optimal regret scales as $\tilde{\Theta}\big(\sqrt{dT}\big)$ for both known and unknown $C$;
    \item For $C = \Omega( \frac{\sqrt{T}}{\log T} )$, we do not provide any sublinear regret bound for when $C$ is unknown, but Theorem \ref{thm:lower_C} shows that, in fact, such a bound cannot be expected for $C = \Omega( \sqrt T )$ unless the uncorrupted regret increases significantly.
    \item For $C$ in between the previous two dot points (e.g., $C = \Theta(T^a)$ with $\frac{1}{4} < a < \frac{1}{2}$), our upper bound for unknown $C$ exhibits strictly higher scaling than the uncorrupted regret (due to the $C^2$ term), and it remains open as to what extent this is unavoidable.
\end{itemize} 


\vspace*{-1ex}
\section{Greedy Algorithm in the Contextual Setting} \label{sec:contextual}
\vspace*{-1ex}

In this section, we consider a $k$-arm linear contextual bandit problem with a single unknown $d$-dimensional parameter vector $\theta \in \RR^d$ (e.g., see \cite{kannan_smoothed_2018}). In each round $t$, contexts $a_{1,t}, \dots, a_{k,t}$ are presented to the learner, each in $\RR^d$ and associated to one action. The learner then chooses an action indexed by $I_t \in \lbrace 1,\dots, k \rbrace$ and observes the corrupted reward  
\begin{equation}
  Y_t = \langle \theta, a_{I_t,t} \rangle + \epsilon_t + c_t(a_{I_t,t}),
\end{equation}
where the same assumptions from Section~\ref{sec:problem_setup} hold for both $(\epsilon_t)_{t=1}^T$ and $c_t(\cdot)$ (with attack budget $C$), and $\| \theta \|_2 \leq 1$.  Similar to \eqref{eq:regret_def}, the cumulative regret is 
%
$R_T = \sum_{t=1}^T \max_{i \in \lbrace 1, \dots, k \rbrace} \langle \theta, a_{i,t} - a_{I_t,t} \rangle.$

In general, the introduction of contexts may significantly complicate the problem, with algorithms such as the one in Section \ref{sec:algorithm_and_regret_bound} being difficult to extend, particularly with unknown $C$.  However, perhaps surprisingly, a line of recent works has demonstrated that simple exploration-free greedy methods can provably work well (in the non-corrupted setting) under mildassumptions on the contexts.  These assumptions amount to kinds of {\em context diversity} \cite{bastani_mostly_nodate,kannan_smoothed_2018,raghavan_externalities_2018} ensuring that the collected samples are sufficiently informative for learning $\theta$ accurately.

Most related to this paper is \cite{kannan_smoothed_2018}, who analyze the greedy algorithm in the case that arbitrary context vectors undergo small random perturbations.   Motivated by these results, we investigate the performance of the greedy algorithm under the same assumption on the contexts, but with the addition of adversarial attacks.  Our main finding is that the context diversity assumption not only removes the need for explicit exploration \cite{kannan_smoothed_2018}, but also automatically inherits near-optimal robustness to adversarial attacks, with no need to know the attack budget $C$.


\textbf{Context generation.} In more detail, the setup of \cite{kannan_smoothed_2018} is introduced as follows:  An arbitrary tuple $\mu_{1,t}, \dots ,\mu_{K,t}$ of mean context vectors is given (possibly selected by an adaptive adversary based on the history of contexts, actions, and rewards), such that $\|\mu_{i,t}\|_2\leq 1$ for all $i,t$. For every available action, the context vector is then generated as $a_{i,t} = \mu_{i,t} + \xi_{i,t}$, where the random perturbation vectors $\xi_{i,t}$ are drawn independently 
from some zero-mean distributions $D_{1,t}, \dots, D_{K,t}$. We consider perturbations that are {\em $(r,\delta)$-bounded} for some $r\leq 1$ according to the following definition \cite{kannan_smoothed_2018}:
\begin{equation}
  \mathbb{P} [ \| \xi_{i,t}\|_{\infty} \leq r \text{ for all arms $i$ and rounds $t$}] \geq 1 - \delta.
\end{equation}

As outlined above, we are interested in the diversity of samples collected by the greedy algorithm (defined below). The main idea is that the observed contexts should cover all directions in order to enable good estimation of the latent vector $\theta$. Consequently, we make use of the notion of diversity from \cite{kannan_smoothed_2018}, which takes into account that the learner observes rewards for contexts that are selected greedily and thus only observes a conditional distribution of contexts. Specifically, following \cite{kannan_smoothed_2018}, a distribution $D$ is called {\em $(r, \lambda_0)$-diverse} with parameters $r > 0$ and $\lambda_0 > 0$ if, for $a = \mu + \xi$ with $\xi \sim D$ and any $\mu \in \RR^d$, it holds for all $\htheta \in \RR^d$ and $\hat{b} \in \RR$ satisfying $\hat{b} \leq r \|\htheta\|_2$ that
\begin{equation}
  \lambda_{\min} \Big( \mathbb{E}_{\xi \sim D}\big[aa^{T} \big|\; \htheta^T \xi \geq \hat{b} \big] \Big) \geq \lambda_0. \label{eq:diverse}
\end{equation}
The overall perturbations are {\em $(r, \lambda_{0})$-diverse} if the distributions $D_{i,t}$ are $(r, \lambda_{0})$-diverse for all $i$ and $t$.

This diversity condition is the main component in \cite{kannan_smoothed_2018} for proving that the minimum eigenvalue of the empirical covariance matrix $\lambda_{\min}(\sum_{\tau=1}^{t} a_{I_{\tau},\tau} a_{I_{\tau},\tau}^{T})$ grows linearly with $t$.  In Lemma~\ref{lemma:estimator_confidence_greedy} (Appendix \ref{sec:pf_context}), we demonstrate that this is the main quantity that has an impact on the accuracy of the estimator of $\theta$, and in turn, on the regret bounds in the corrupted setting.

\textbf{Greedy algorithm.}
In round $t$, the greedy algorithm (see Algorithm~\ref{alg:cg}) receives a set of contexts $\lbrace a_{1,t}, \dots, a_{k,t} \rbrace$, and chooses the best action according to the least squares estimate of $\theta$:
\begin{gather} \label{eq:greedy}
  I_t = \argmax_{i \in \lbrace 1, \dots, K \rbrace} \langle \hat{\theta}_t, a_{i,t} \rangle, 
    \\ \htheta_{t} = \argmin_{\theta'} \sum_{\tau=1}^{t-1} (\langle \theta', a_{I_{\tau},\tau} \rangle - Y_\tau)^2.
\end{gather}





Our regret bound for this setup is stated as follows, and proved in Appendix \ref{sec:pf_context}.

\begin{theorem} \label{thm:greedy_regret_bound_thm}
    Suppose that $\|a_{i,t}\|_2 \leq 1$ for all $i,t$, the random context perturbations are ($r$, $1/T$)-bounded and $(r,\lambda_0)$-diverse with $r \leq 1$, the reward noise is $1$-sub-Gaussian, and the attack budget is $C \geq 0$. Then with probability at least $1 - \delta$, the greedy algorithm has regret bounded by
  \begin{align}
    &R_T = O \bigg(\tfrac{1}{\lambda_0} \Big(\sqrt{dT \log \big(\tfrac{dT}{\delta}\big)} + C \log T + \log\big(\tfrac{dT}{\delta}\big)\Big)  \nonumber \\
    &\hspace*{5cm} + \sqrt{\log(\tfrac{k}{\delta})} \bigg).
  \end{align}
\end{theorem}

Under the mild assumptions $\delta = e^{-O(dT)}$ and $\frac{k}{\delta} = e^{O(dT)}$, this bound simplifies to 
\begin{equation} \label{eq:nicer_version}
    R_T = O \bigg(\tfrac{1}{\lambda_0} \Big(\sqrt{dT \log \big(\tfrac{Td}{\delta}\big)} + C \log T\Big) \bigg).
\end{equation}
In addition, when $C = 0$, Theorem \ref{thm:greedy_regret_bound_thm} reduces to the result of \cite{kannan_smoothed_2018}.  The additional $\frac{1}{\lambda _0}C \log T$ term is essentially optimal when $\lambda_0 = \Theta(1)$, since a simple argument from \cite{lykouris2018stochastic} gives an $\Omega(C)$ lower bound (see Appendix \ref{sec:pf_lower}). In Corollary~\ref{corr:greedy} (Appendix \ref{sec:pf_context}), we specialize Theorem \ref{thm:greedy_regret_bound_thm} to the case that the perturbations are Gaussian, i.e., every $\xi_{i,t}$ is drawn independently from $\mathcal{N}(0,\eta^2 I)$, and show that the greedy algorithm has sublinear regret in the low-$\eta$ regime.

Theorem \ref{thm:greedy_regret_bound_thm} indicates that the greedy algorithm can be robust despite being extremely simple, having no explicit built-in mechanism for combating robustness, and having no knowledge $C$.  A caveat to this is the $\frac{1}{\lambda_0}$ dependence, indicating that the regret can increase significantly when the contexts are not sufficiently diverse.

\begin{figure*}
    \centering
    \includegraphics[width=0.245\textwidth]{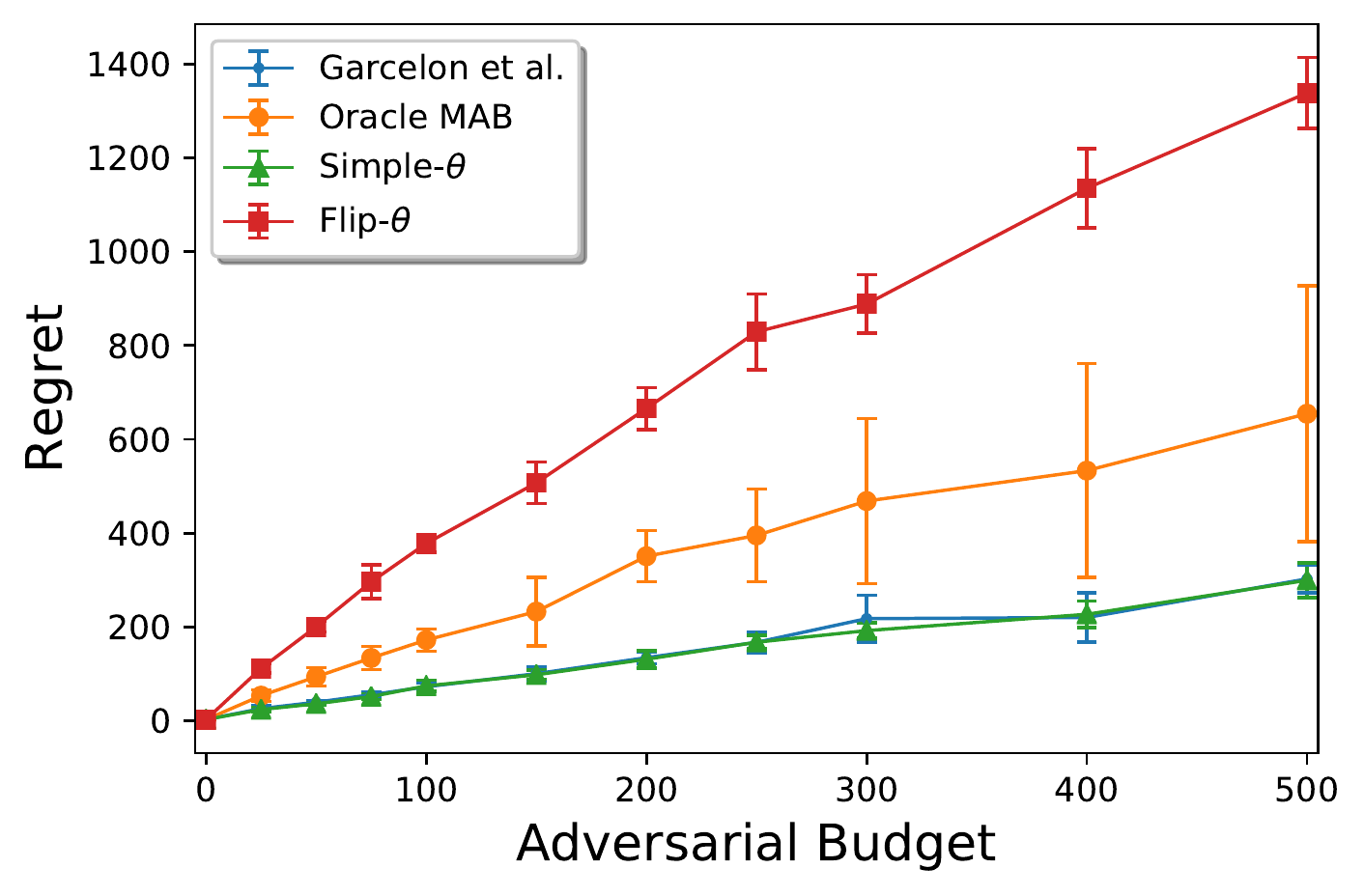}
    \includegraphics[width=0.245\textwidth]{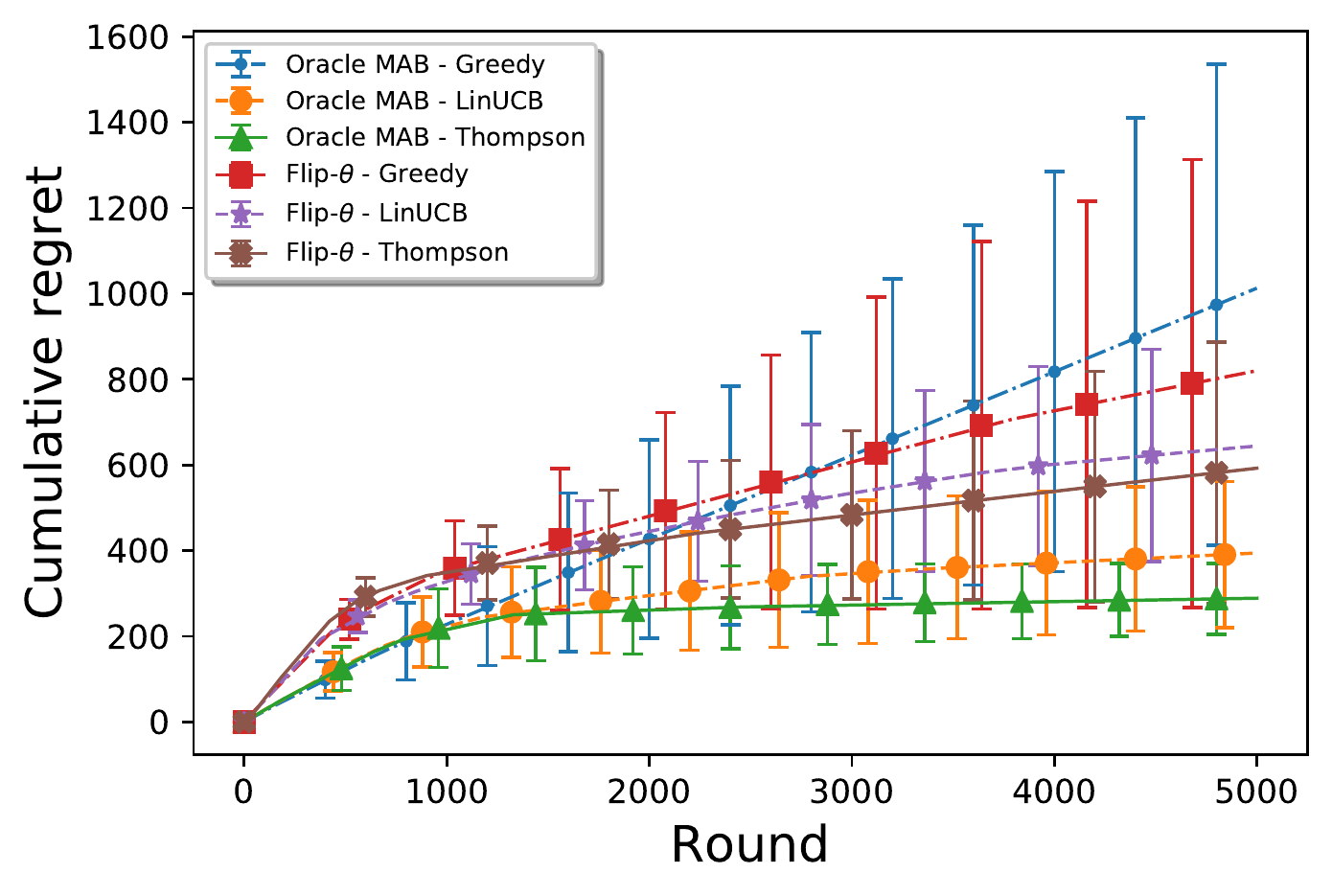}
    \includegraphics[width=0.245\textwidth]{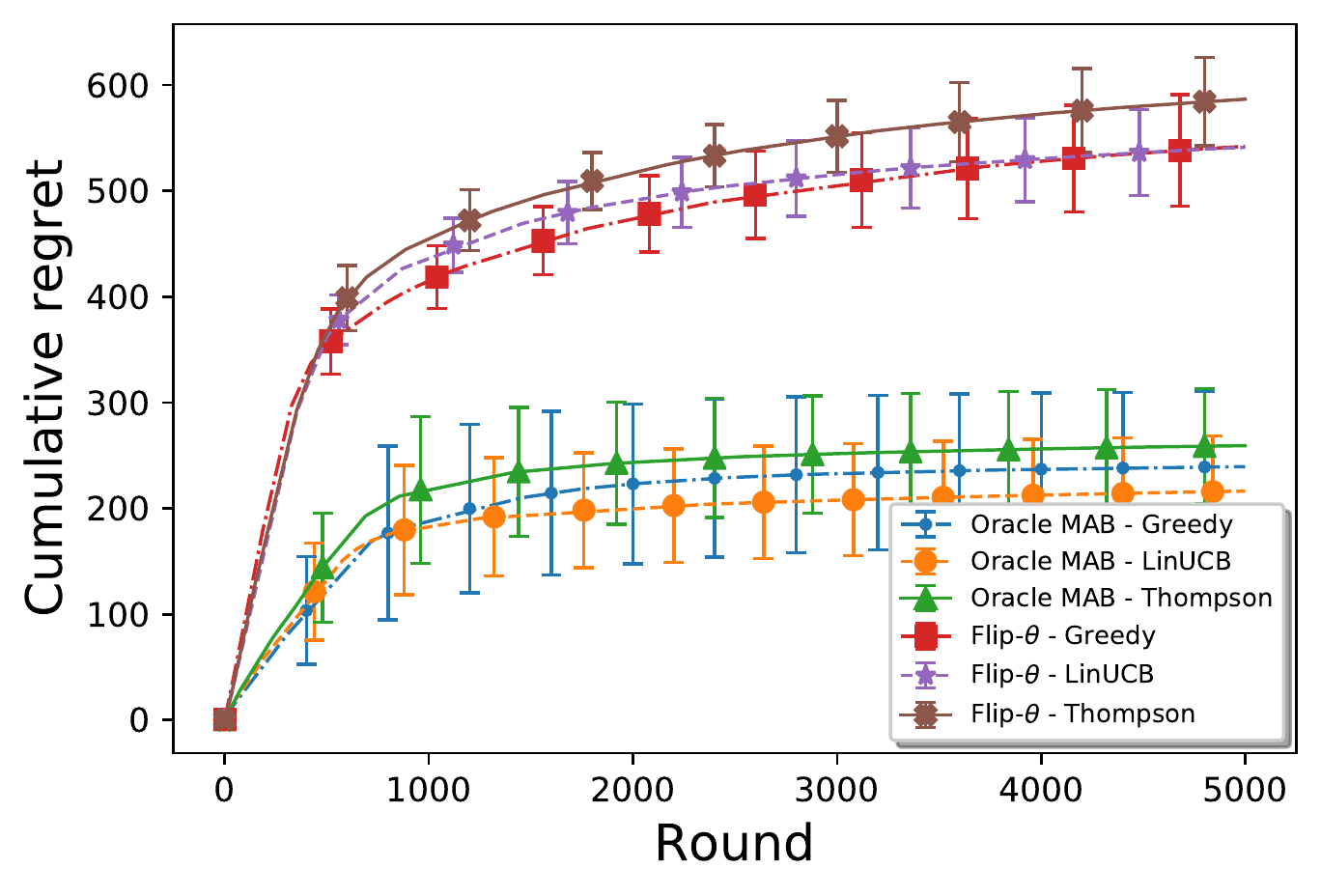}
    \includegraphics[width=0.245\textwidth]{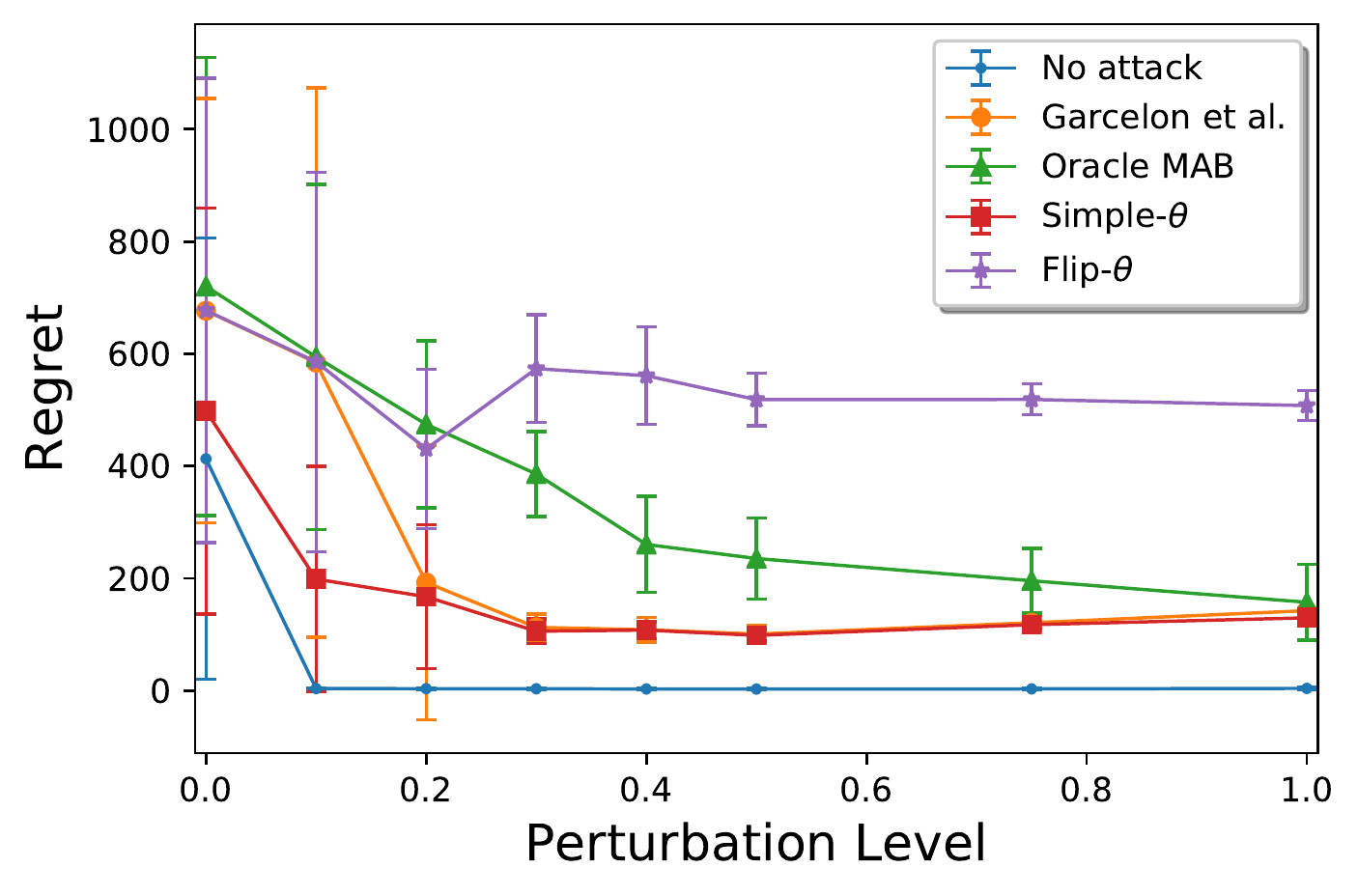}
    \caption{Contextual synthetic experiment: (Left) Regret at time $T = 3500$ as a function of $C$ with $\eta = 0.5$; (Middle Two) Regret as a function of time with $\eta = 0$ and $\eta = 0.5$; (Right) Performance of Greedy at time $T=3500$ with $C = 150$ and varying $\eta$.  }
    \label{fig:attacks}
    \vspace*{-2ex}  
\end{figure*} 

\begin{figure*}
    \centering
    \includegraphics[width=0.245\textwidth]{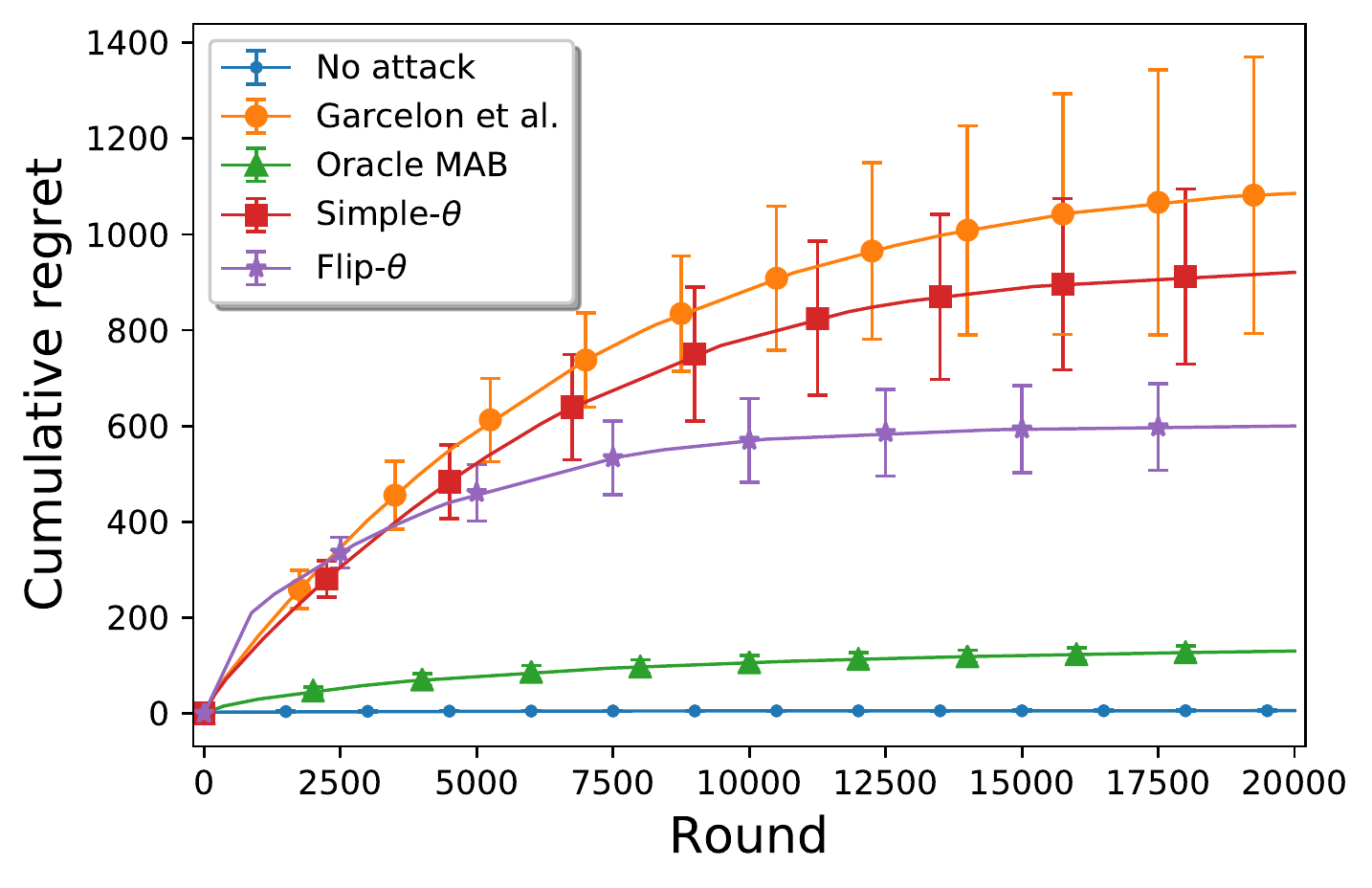}
    \includegraphics[width=0.245\textwidth]{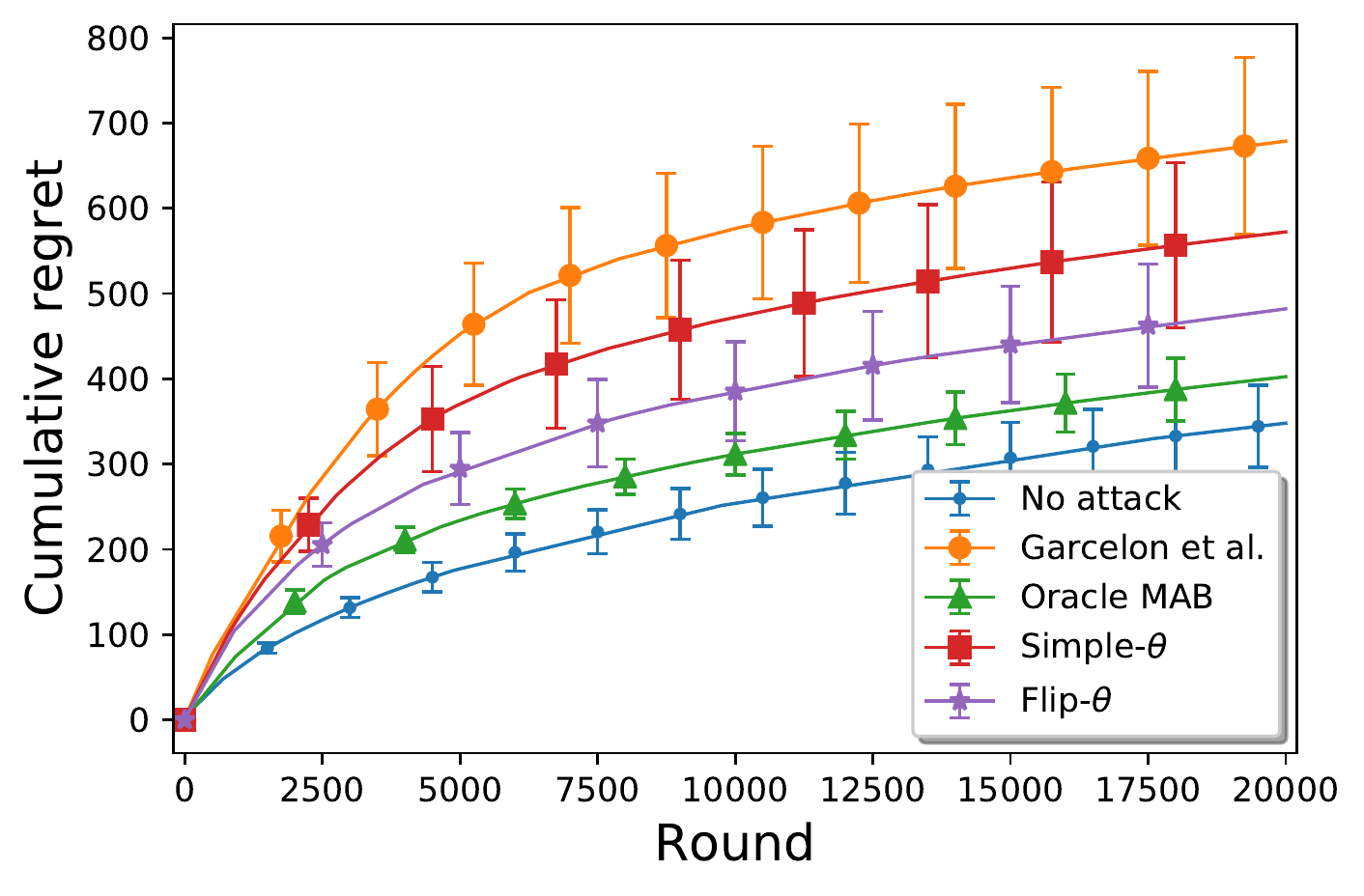}
    \includegraphics[width=0.245\textwidth]{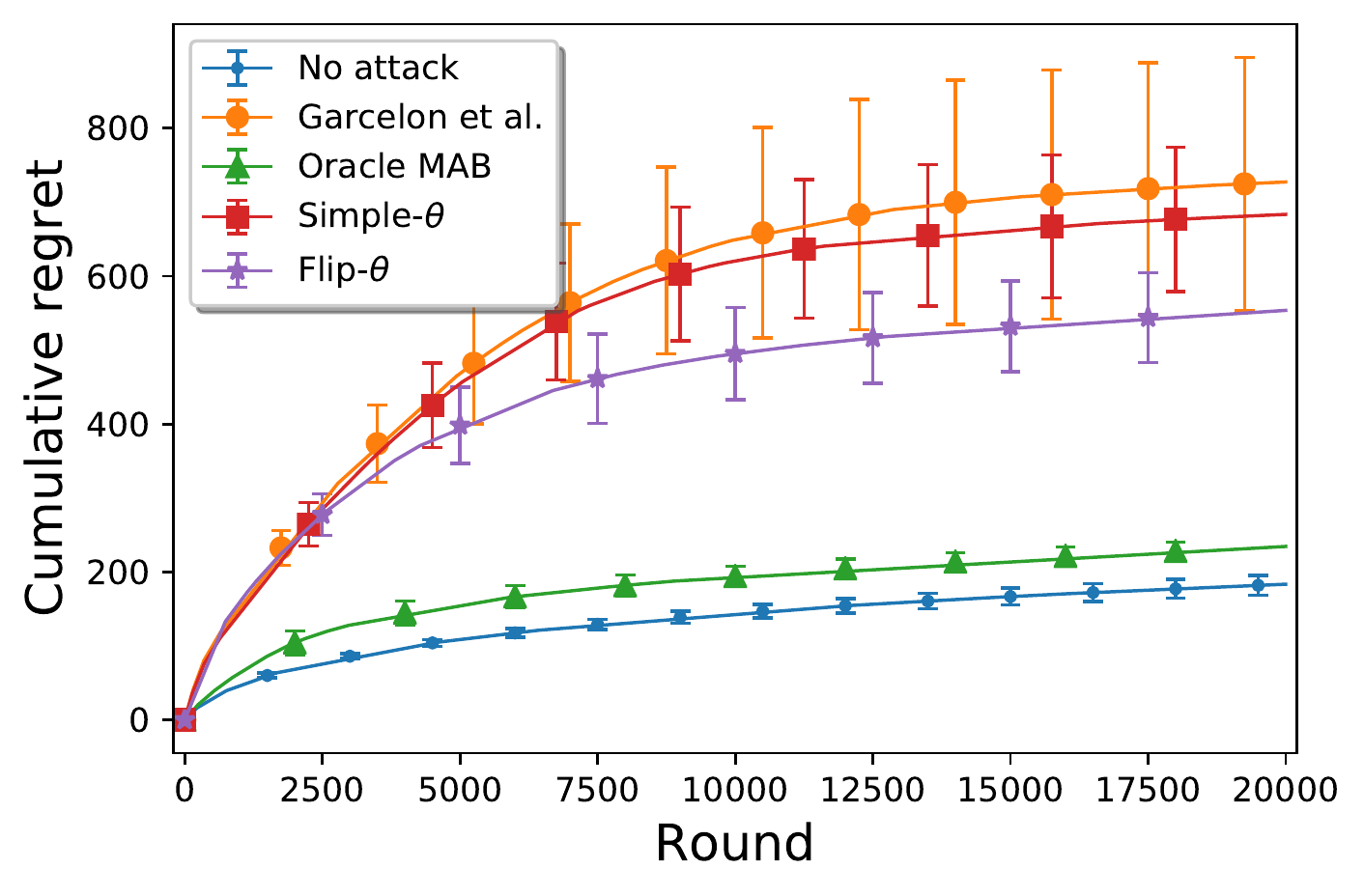}
    \includegraphics[width=0.245\textwidth]{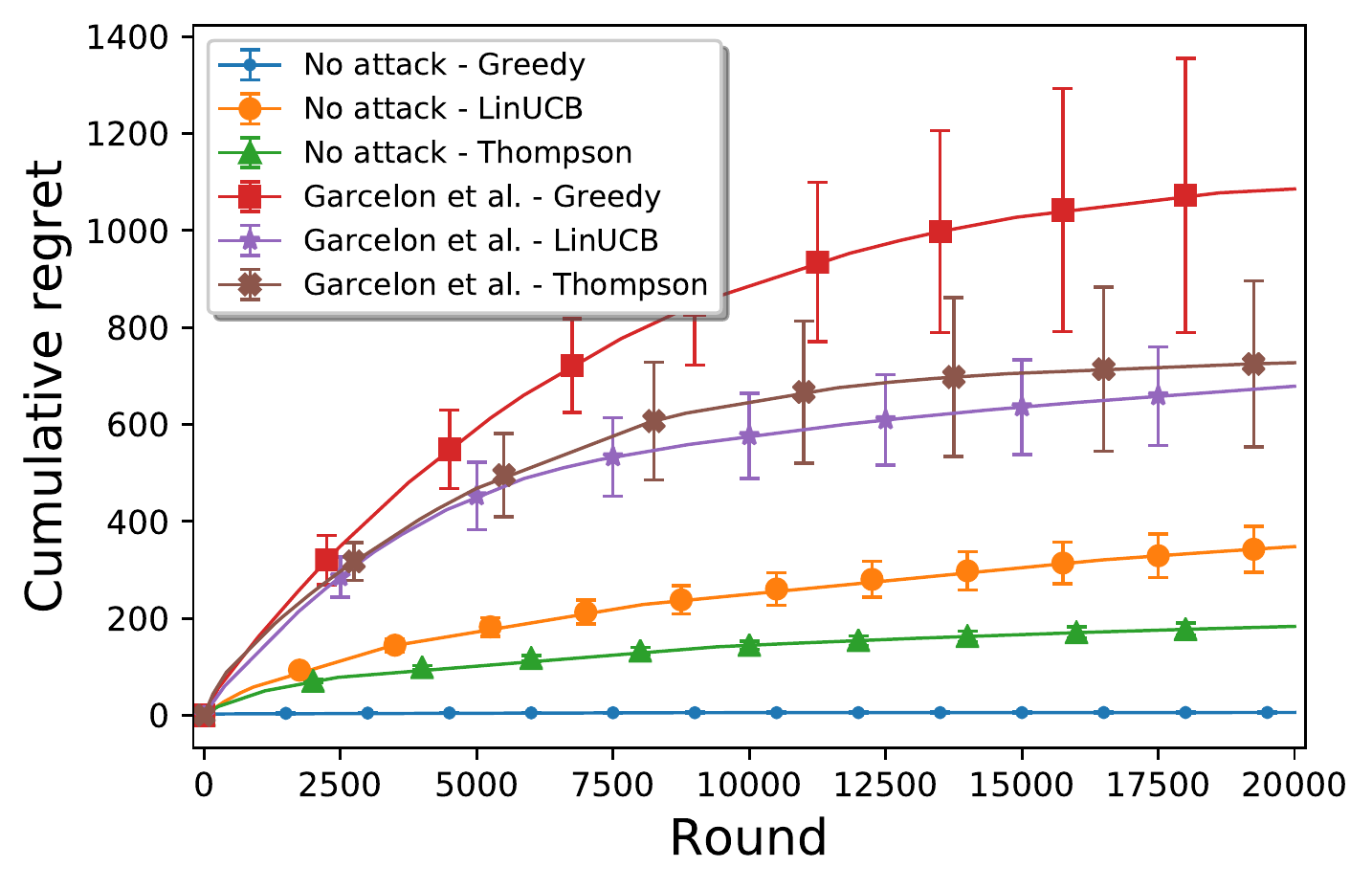}
    \caption{MovieLens experiment: (Left 3) Regret as a function of time with $C = 150$ for Greedy, LinUCB, and Thompson sampling; (Right) Regret of all algorithms under the Garcelon {\em et al.}~attack.}
    \label{fig:regret_time}
    \vspace*{-3ex}  
\end{figure*} 

\vspace*{-1ex}
\section{Experiments}
\vspace*{-1ex}

In this section, we evaluate the performance of the algorithms studied in this paper, along with the baselines LinUCB \cite{li2010contextual,lattimore2018bandit} and Thompson sampling \cite{agrawal2013thompson}.\footnote{We use LinUCB as described in \cite[Sec.~19.2]{lattimore2018bandit} with least-squares regularization parameter $\lambda = 1$ and confidence parameter $\delta = 0.1$, and Thompson sampling \cite{agrawal2013thompson} uses an i.i.d.~Gaussian prior with variance $0.5$.}  We consider both the robust PE algorithm and the contextual greedy algorithm, starting with the latter. 

\vspace*{-1ex}
\subsection{Choices of Attacks} \label{sec:attacks_main}
\vspace*{-1ex}

We consider the following attack algorithms, each depending on a target arm $a_{\rm target}$ and/or a target parameter vector $\theta_{\rm target}$.  These are briefly outlined as follows, with more details in Appendix \ref{sec:attacks}:
\begin{itemize}
    \item {\bf Garcelon {\em et al.}~attack.} This attack is a minor modification of that of \cite{garcelon_adversarial_2020}, leaving pulls from $a_{\rm target}$ uncorrupted, while pushing all other rewards down to the minimum value.
    \item {\bf Oracle MAB attack.} This attack from \cite{jun_adversarial_nodate} pushes the reward of any $a \ne a_{\rm target}$ to some margin $\epsilon_0$ below that of $a_{\rm target}$, or leaves the reward unchanged if such a margin is already met.
    \item {\bf Simple $\theta$-based attack.} This attack acts in the same way as that of Garcelon {\em et al.}, but with $a_{\rm target}$ always chosen as $\argmax_{a}\langle a, \theta_{\rm target}\rangle$.  This is equivalent to that of \cite{garcelon_adversarial_2020} in the non-contextual setting, but otherwise may differ due to $a_{\rm target}$ varying with time.
    \item {\bf Flip-$\theta$ attack.} This attack simply flips the reward from $\langle \theta, a \rangle$ to $\langle -\theta, a \rangle$.
\end{itemize}
Note that the terminology ``oracle'' refers to attacks that use knowledge of $\theta$, which we assume to be permitted in this paper (the Flip-$\theta$ attack also falls in this category).  We set $a_{\rm target}$ to be the first arm, which will have the same effect as choosing any fixed arm (since our arm feature vectors will be generated in a symmetric manner).  In addition, we let $\theta_{\rm target}$ be uniform on the unit sphere in the simple $\theta$-based attack, and set $\epsilon_0 = 0.01$ in the Oracle MAB attack. 

\vspace*{-1ex} 
\subsection{Contextual Setting} \label{sec:exp_context}
\vspace*{-1ex}

\begin{figure*}
    \centering
    \includegraphics[width=0.32\textwidth]{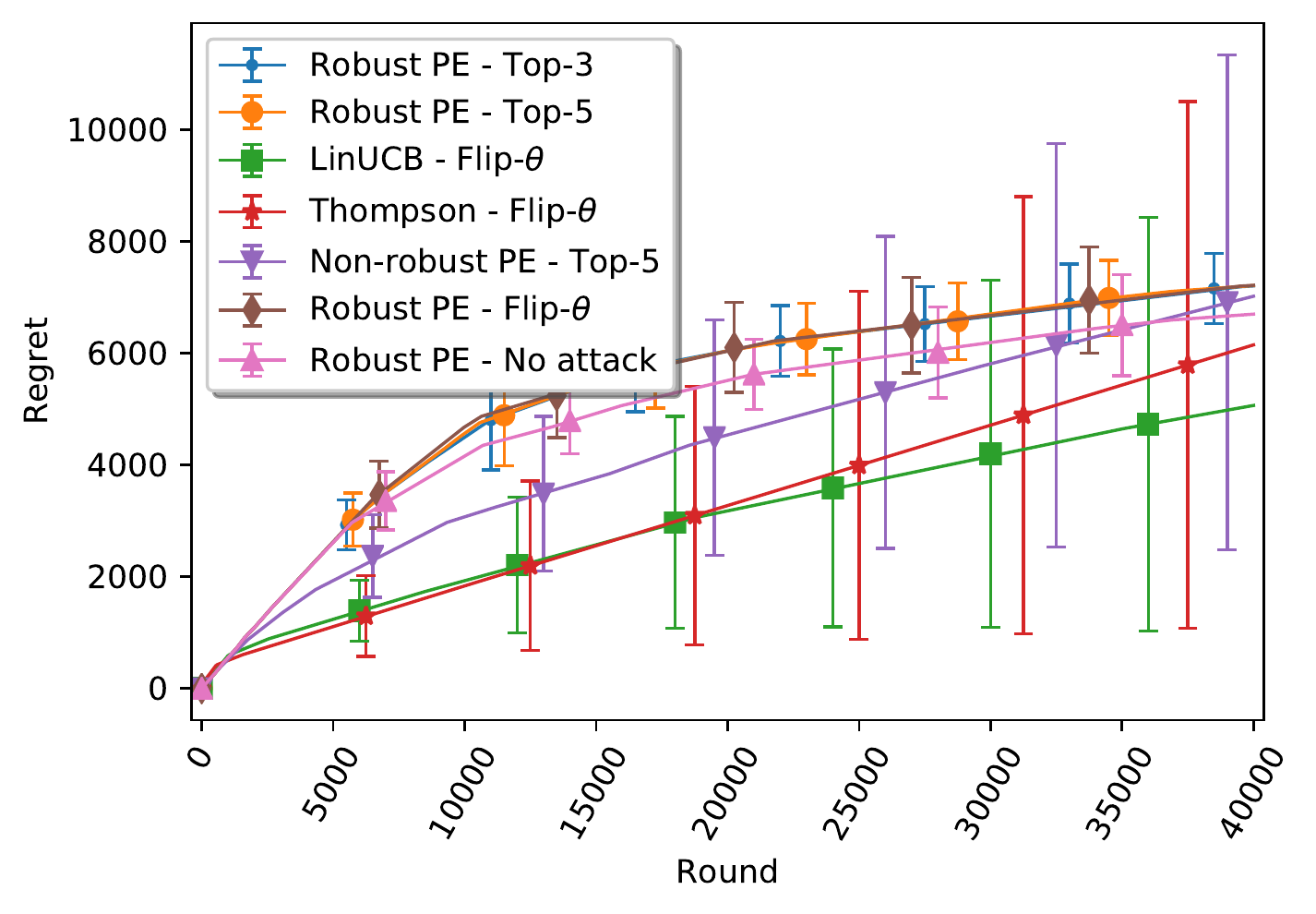}
    \includegraphics[width=0.32\textwidth]{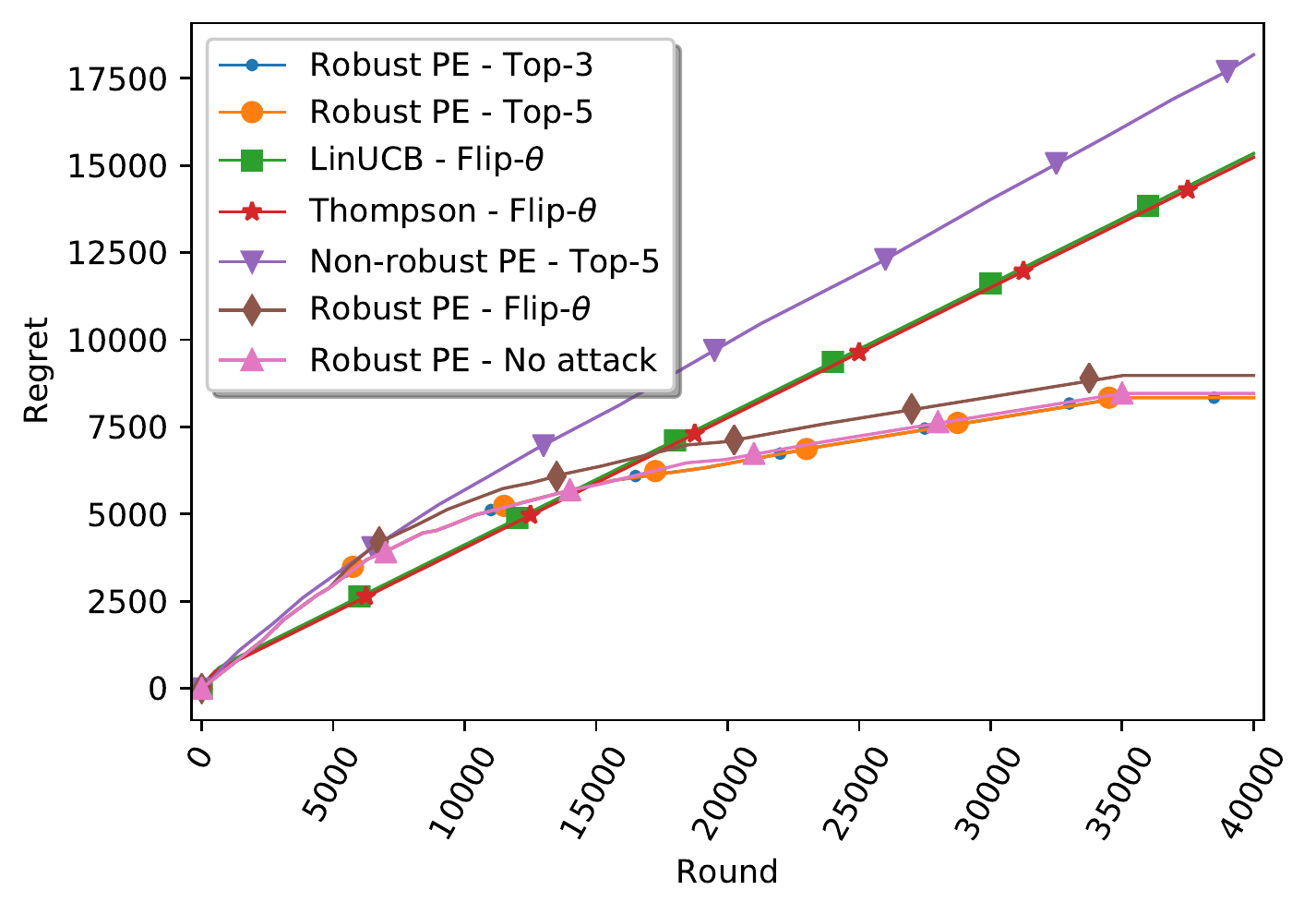}
    \includegraphics[width=0.32\textwidth]{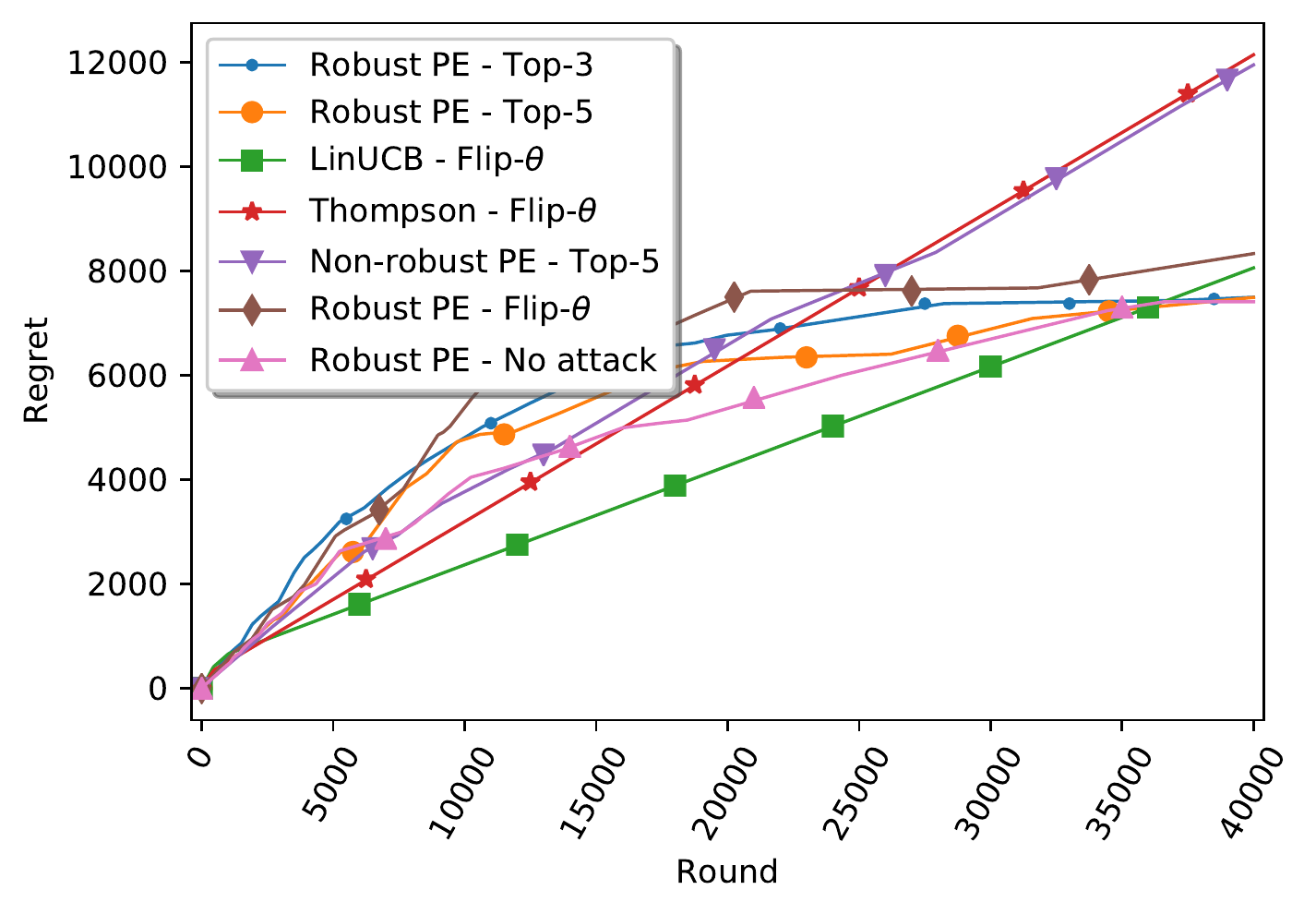}
    \caption{Non-contextual synthetic experiment with 10 trials: (Left) Average regret as a function of time; (Middle) Worst run among 10; (Right) Second-worst run among 10.}
    \label{fig:noncontext}
    \vspace*{-2.5ex}  
\end{figure*} 

{\bf Synthetic Experiment.} In this experiment, we consider the contextual case with contexts having uniform entries and Gaussian perturbations with variance $\eta^2 > 0$; see Appendix \ref{sec:more_exp} for the full details.  We consider $k=25$ arms, $T = 5000$ rounds, and attack budget $C = 50$. At each time instant, we plot the cumulative regret averaged over $10$ trials, and error bars indicate one standard deviation.  In Appendix \ref{sec:more_exp}, we provide analogous plots and discussion when $C=150$.

In Figure \ref{fig:attacks} (Left), we plot the regret of Greedy at $T=3500$ as a function of $C$ with $\eta = 0.5$.  We observe a linear increase, which is in agreement with our theory.  Analogous plots for LinUCB, Thompson sampling, and $\eta \in \{0.2,0.5\}$ can be found in Appendix \ref{sec:more_exp}.  The middle two plots in  Figure \ref{fig:attacks} show the regret as a function of time with the two most effective attacks, with $\eta = 0$ and $\eta = 0.5$.  We see that the regret curves are still increasing linearly under the Flip-$\theta$ attack by time $T = 5000$ when $\eta = 0$, whereas they are nearly flat when $\eta = 0.5$.  While our theory only supports the robustness of Greedy, these experiments suggest that LinUCB and Thompson sampling may also enjoy similar robustness under context diversity.  Finally, Figure \ref{fig:attacks} (Right) plots the regret of Greedy at $T = 3500$ as a function of $\eta$ when $C = 150$.  We observe that once $\eta$ moves past a certain level, the performance remains fairly consistent, with a general (but not definitive) trend of decreasing regret.  The greatest difference is at $\eta = 0$, particularly when the standard deviation is considered.


{\bf MovieLens Experiment.} We use the MovieLens-100K dataset in a similar manner to \cite{bogunovic2018adv}; see Appendix \ref{sec:more_exp} for details.  In each trial, we select a uniformly random user and treat the 1682 movies as possible contexts.  At each time instant, $k = 30$ of these movies are chosen uniformly at random and presented as the contexts.  Hence, a subset of the movie vectors form the contexts, and a fixed user vector forms $\theta$.  We set $T = 20000$ and $C = 150$, and we plot the regret averaged over 10 trials (each corresponding to a different user).

In Figure \ref{fig:regret_time}, we plot the regret as a function of time, for Greedy, LinUCB, and Thompson sampling.  Despite the lack of explicit context perturbation in this experiment, we see that the algorithms are again able to recover from the attacks, suggesting that the various movies in the data set are sufficiently diverse.  On the other hand, we do not claim the attacks here to be optimal, and it is possible that stronger attacks may incur linear regret.  In Figure \ref{fig:regret_time} (Right), we plot all three algorithms under the strongest attack and under no attack.  We see that Greedy has very low regret when there is no attack, but has slightly higher regret when attacked.

\vspace*{-1ex}
\subsection{Non-Contextual Setting} \label{sec:exp_noncontext}
\vspace*{-1ex}

We now turn to experiments for the robust PE algorithm (Algorithm \ref{alg:cpe}), with some minor practical changes detailed in Appendix \ref{sec:more_exp}.  We use the above synthetic experimental setup with the context perturbations removed (i.e., $\eta = 0$), and with $d = 5$, $k=50$, $T = 40000$, and $C = 150$.  For comparison, we also include non-robust PE, which removes the second term in \eqref{eq:retain_arms_condition}.

For LinUCB, Thompson sampling, and non-robust PE, we continue to attack right from the start.  However, for robust PE, this is a poor attack strategy, since the algorithm initially uses a very stringent condition for elimination.  Instead, following insight from the proof of Theorem \ref{thm:regret_bound2}, we start the attack at the first epoch for which $\hat{C}_h < C$.
We consider the Flip-$\theta$ attack of Section \ref{sec:attacks_main}, as well as an additional {\em Top-$N$ attack} targeted at eliminating good arms: Whenever any of the top $N$ remaining arms are pulled, push the reward to $-1$.  We consider both $N=3$ and $N=5$.  We focus on the case of unknown $C$ here, and present similar plots for known $C$ in Appendix \ref{sec:more_exp}.

In Figure \ref{fig:noncontext} (Left), we see that the average regret of all algorithms is similar by the end of the time horizon; however, an inspection of the error bars reveals that this is not the full story.  In particular, the regret of LinUCB and Thompson sampling vary considerably depending on whether the attack was successful or not, whereas robust PE exhibits much lower variation.  To highlight this, we plot the regret from the worst and second-worst runs out of 10 (as measured at time $T$) in Figure \ref{fig:noncontext} (Middle) and Figure \ref{fig:noncontext} (Right).   In Appendix \ref{sec:more_exp}, we provide analogous plots in the case of 40 trials, showing the worst 4-out-of-40 runs and observing similar behavior to Figure \ref{fig:noncontext}.

We see that LinUCB and Thompson sampling visibly have linear regret, whereas the regret of robust PE flattens out by the end of the time horizon even for these worst-2-of-10 curves, indicating better {\em high-probability} behavior.  In contrast, these results suggest the possibility of algorithms with improved {\em finite-time} performance guarantees, which was not the focus of our~work.

\vspace*{-1ex}
\section{Conclusion}
\vspace*{-1ex}

We have considered the linear stochastic problem in the presence of adversarial attacks/corruptions. We provided novel algorithms in both the standard and contextual settings that are provably robust against such attacks.  We demonstrated near-optimal regret bounds in all cases, and to our knowledge, we are the first to do so in each case.
A possible direction for future work is to consider a setting in which both rewards and contexts can be altered by the adversary subject to a limited attack budget.

\vspace*{-1ex}
\section*{Acknowledgments}
\vspace*{-1ex}
This project has received funding from the European Research Council (ERC) under the European Unions Horizon 2020 research and innovation programme grant agreement No 815943 and ETH Z\"urich Postdoctoral Fellowship 19-2 FEL-47. J. Scarlett was supported by the Singapore National Research Foundation (NRF) under grant number R-252-000-A74-281.
\bibliographystyle{myIEEEtran}
\bibliography{refs}
	
\appendix

\section{Proofs for Section \ref{sec:algorithm_and_regret_bound} (Robust Phased Elimination Algorithm)} \label{sec:pf_pe}

\subsection{Single Epoch Analysis (Known Corruption Budget)}

In this section, we consider a single fixed epoch indexed by $h$.  Since the analysis here holds for any epoch, we omit the subscript $(\cdot)_h$ throughout; in particular, $\htheta = \htheta_h$ and $\Gamma = \Gamma_h$ are as given in \eqref{eq:estimators2}, $\A=\A_h$ is the active set of arms, $u(a) = u_h(a)$ is the number of times $a$ is played, and $u = u_h$ is the total length of the epoch.

\begin{lemma} \label{lemma:mean_diff}
  For any action $b \in \A$, the following holds with probability at least $1-2\delta$:
  \begin{equation} \label{eq:conf_bound}
    |\langle b, \htheta - \theta \rangle| \leq \|b\|_{\Gamma^{-1}} \sqrt{2 \log (\tfrac{1}{\delta}}) + \frac{C}{m \nu} \sqrt{\sum_{a\in \A} u(a) \|b\|_{\Gamma^{-1}}}.
  \end{equation} 
  \begin{proof}
  We recall the definition of the set $\T(a) = \big\lbrace  s \in \lbrace 1,\dots,u \rbrace : A_s = a \big\rbrace$, whose cardinality is given by $u(a)$. We characterize the considered estimator of $\theta$ as follows:
    \begin{align}
      \hat{\theta} &= \Gamma^{-1} \sum_{t=1}^{u} A_t u(A_t)^{-1} \sum_{s \in \T(A_t)}  Y_s \label{eq:considered1} \\
      &= \Gamma^{-1} \sum_{t=1}^{u} A_t u(A_t)^{-1} \Big(\sum_{s \in \T(A_t)} \big(\langle \theta, A_s \rangle + \epsilon_s + c_s(A_s)\big) \Big) \label{eq:considered2} \\
      &= \Big(\Gamma^{-1} \sum_{t=1}^{u} A_t A_t^T \theta \Big) + \Big(\Gamma^{-1} \sum_{t=1}^{u} A_t u(A_t)^{-1} \sum_{s \in \T(A_t)}  \epsilon_s \Big)  \nonumber \\
      &\quad + \Big(\Gamma^{-1} \sum_{t=1}^{u} A_t u(A_t)^{-1} \sum_{s \in \T(A_t)} c_s(A_s) \Big) \label{eq:considered3} \\
      &= \theta + \Big(\Gamma^{-1} \sum_{t=1}^{u} A_t \epsilon_t \Big) + \Big(\Gamma^{-1} \sum_{t=1}^{u} A_t u(A_t)^{-1} \sum_{s \in\T(A_t)} c_s(A_s) \Big), \label{eq:considered4}
    \end{align}
    where \eqref{eq:considered2} uses the decomposition of $Y_s$ into the reward/noise/corruption, \eqref{eq:considered3} uses $u(a) = |\T(a)|$ and the fact that all $s \in \T(A_t)$ have $A_s = A_t$, and \eqref{eq:considered4} uses $\Gamma = \sum_{a \in \A} u(a) a a^T = \sum_{t=1}^{u} A_t A_t^T$.

    By \eqref{eq:considered4}, for any action $b \in \A$ (or more generally $b \in \A_0$), we have
    \begin{align} \label{eq:separation}
      \big|\langle b, \htheta - \theta \rangle\big| \leq \Bigg|b^T\Gamma^{-1} \sum_{t=1}^{u} A_t \epsilon_t \Bigg| + \Bigg| b^T\Gamma^{-1} \sum_{t=1}^{u} A_t u(A_t)^{-1} \sum_{s \in \T(A_t)} c_s(A_s) \Bigg|.
    \end{align}
    We proceed by bounding the two terms separately. The second term can be rewritten as follows:
    \begin{equation}
      \Bigg| b^T\Gamma^{-1} \sum_{t=1}^{u} A_t u(A_t)^{-1} \sum_{s \in \T(a)} c_s(A_s) \Bigg| = \Bigg| \sum_{a \in \A, u(a)\neq 0} \frac{C_a}{u(a)} u(a) b^T\Gamma^{-1}a  \Bigg|,
    \end{equation}
    where we use $C_a$ to denote $\sum_{s \in \T(a)} c_s(a)$, i.e., the sum of corruptions for arm $a$ in the epoch, and we keep the factor $\frac{u(a)}{u(a)}$ for convenience in what follows.

    Next, we have
    \begin{align}
      \Bigg| \sum_{a \in \A, u(a)\neq 0} \frac{C_a}{u(a)} u(a) b^T\Gamma^{-1}a  \Bigg|
      &\leq \sum_{a \in \A, u(a)\neq 0} \frac{C}{u(a)} u(a) \Big|b^T\Gamma^{-1}a \Big| \label{eq:ctb1} \\
      &\leq  \frac{C}{m \nu} \sum_{a \in \A} u(a) \Big|b^T\Gamma^{-1}a \Big| \label{eq:ctb2}  \\
      &\leq \frac{C}{m \nu} \sqrt{\Big(\sum_{a \in \A} u(a)\Big) b^T \sum_{a \in \A} u(a) \Gamma^{-1} aa^T \Gamma^{-1} b} \label{eq:ctb3} \\
    &= \frac{C}{m \nu} \sqrt{ \sum_{a\in \A} u(a) \|b\|_{\Gamma^{-1}}} \label{eq:corruption_term_bound},
    \end{align}
    where \eqref{eq:ctb1} uses the triangle inequality and $|C_a| \le C$ for every $a$, \eqref{eq:ctb2} holds since $u(a) \geq \nu m$ by the choice of $u(a)$, \eqref{eq:ctb3} follows by multiplying and dividing by $\sum_{\bar{a}\in\A}u(\bar{a})$ and applying $\mathbb{E}[|Z|] \le \sqrt{\mathbb{E}[Z^2]}$ with the distribution $\frac{u(a)}{\sum_{\bar{a}\in\A} u(\bar{a})}$, and \eqref{eq:corruption_term_bound} follows by taking the first $\Gamma^{-1}$ term outside the sum and applying the definition of $\Gamma = \Gamma_h$ from \eqref{eq:estimators2}.

    Since $(\epsilon_t)_{t=1}^u$ are independent and $1$-sub-Gaussian, the first term in \eqref{eq:separation} is bounded via standard concentration results: From \cite[Eq. (20.2)]{lattimore2018bandit}, with probability at least $1-2\delta$, we have
    \begin{equation} \label{eq:noise_term_bound}
      \Big|b^T\Gamma^{-1} \sum_{t=1}^{u} A_t \epsilon_t \Big| \leq \| b\|_{\Gamma^{-1}} \sqrt{2\log(\tfrac{1}{\delta})}.
    \end{equation} 
    Combining the bounds obtained in \eqref{eq:corruption_term_bound} and \eqref{eq:noise_term_bound} completes the proof. 

  \end{proof}  
\end{lemma} 

Next, we characterize the term $\|b\|_{\Gamma^{-1}}^2 = \sqrt{b^T \Gamma^{-1} b}$ appearing in \eqref{eq:conf_bound}.

\begin{lemma} \label{lemma:weighted_norm_b}
  For any arm $b \in \A$, it holds that
  \begin{equation}
    \|b\|_{\Gamma^{-1}}^2 \leq \frac{2d}{m}.
  \end{equation}
\end{lemma}
\begin{proof}
  We have
  \begin{align}
    \|b\|_{\Gamma^{-1}}^2 &= b^T \Gamma^{-1} b \label{eq:bterm1} \\
    &= b^T \Big(\sum_{a \in \A} u(a) aa^T\Big)^{-1} b \label{eq:bterm2} \\
    &= b^T \Big(\sum_{a \in \A} \lceil m  \max\lbrace \zeta(a),\nu \rbrace \rceil aa^T\Big)^{-1} b \label{eq:bterm3} \\
    &\leq b^T \Big(\sum_{a \in \A}  m \zeta(a) aa^T\Big)^{-1} b  \label{eq:bterm4}\\
    &\leq \frac{2d}{m}, \label{eq:bterm5}
  \end{align}
where:
\begin{itemize}[leftmargin=5ex,itemsep=0ex,topsep=0.25ex]
    \item \eqref{eq:bterm2} and \eqref{eq:bterm3} follow from the definitions of $\Gamma = \Gamma_h$ and $u(a) = u_h(a)$ in Algorithm \ref{alg:cpe};
    \item \eqref{eq:bterm4} follows by letting $A = \sum_{a \in \A}  m \zeta(a) aa^T$ and $B=\sum_{a \in \A} \lceil m  \max\lbrace \zeta(a),\nu \rbrace \rceil aa^T$ and noting that $\|b\|_{A^{-1}}^2 \geq \| b\|_{B^{-1}}^2$ whenever $A^{-1} \succeq B^{-1}$, or equivalently $B\succeq A$ (i.e., inversion reverses Loewner orders).;
    \item \eqref{eq:bterm5} follows from $\max_{a \in \A_{h}} \| a \|_{\Gamma(\zeta_h)^{-1}}^2 \leq 2d$ (second step in Algorithm~\ref{alg:cpe}) and the definition $\Gamma(\zeta) = \sum_{a \in \A} \zeta(a) aa^T$. 
\end{itemize}
\end{proof}

Combining the results obtained in Lemmas~\ref{lemma:mean_diff} and~\ref{lemma:weighted_norm_b}, we find that with probability at least $1 - 2\delta$, the following holds for any $b \in \A$:   
\begin{equation} \label{eq:rews}
  |\langle b, \htheta - \theta \rangle| \leq \sqrt{\frac{4d}{m}\log\Big(\frac{1}{\delta}\Big)} + 
  \frac{C}{m \nu} \sqrt{\frac{2d u}{m}}.
\end{equation}

In the following lemma, we bound the total epoch length $u = u_h$ in terms of the quantity $m = m_h$ from Algorithm \ref{alg:cpe}.

\begin{lemma} \label{lemma:epoch_length}
    Let $m_0 = 4d(\log \log d + 18)$, and let $\nu \in (0,1)$ be the truncation parameter. Then, the epoch length in Algorithm~\ref{alg:cpe} is bounded as $u \leq 2m (1 + m_0 \nu)$.
\end{lemma}

\begin{proof}
  We have 
  \begin{align}
    u &= \sum_{a \in \A, \zeta(a) \neq 0} \lceil m \max \lbrace \zeta(a), \nu \rbrace \rceil  \\
      &\leq \sum_{a \in \A, \zeta(a) \neq 0} \big( m \max \lbrace \zeta(a), \nu \rbrace + 1 \big) \\
      &\leq 4d(\log \log d + 18) +  \sum_{a \in \A,\zeta(a) \neq 0} m \max \lbrace \zeta(a), \nu \rbrace \label{eq:ubound4} \\
      &\leq  m +  \sum_{a \in \A,\zeta(a) \neq 0} m \max \lbrace \zeta(a), \nu \rbrace \label{eq:ubound5}  \\
      &\leq 2m  \sum_{a \in \A,\zeta(a) \neq 0} \max \lbrace \zeta(a), \nu \rbrace  \label{eq:ubound6} \\
      &\leq 2m (1 + m_0 \nu), \label{eq:ubound7} 
  \end{align}
  where \eqref{eq:ubound4} uses the support bound in \eqref{eq:solve_near_optimal_design2}, \eqref{eq:ubound5} uses $m_0 =  4d(\log \log d + 18)$ and $m_0 \le m$, \eqref{eq:ubound6} follows since $\sum_{a \in \A,\zeta(a) \neq 0} \max \lbrace \zeta(a), \nu \rbrace \ge \sum_{a} \zeta(a) = 1$ for any $\nu \in (0,1)$, and \eqref{eq:ubound7} uses $\max\lbrace \alpha,\beta \rbrace \leq \alpha + \beta$ for $\alpha,\beta \geq 0$.
\end{proof}

We are now in position to state the main lemma of this section, which follows by combining Lemma~\ref{lemma:epoch_length} with \eqref{eq:rews}, and provides corruption-tolerant confidence bounds.

\begin{lemma} \label{lem:robust_conf}
    In the given (arbitrary) epoch under consideration, with probability at least $1-2\delta$, we have for any $a \in \A$ that
    \begin{equation} \label{eq:mean_diff} 
      |\langle a, \htheta - \theta \rangle| \leq \sqrt{\frac{4d}{m}\log\Big(\frac{1}{\delta}\Big)} + 
      \frac{C}{m \nu} \sqrt{4d(1+ \nu m_0)}.
    \end{equation}
    In addition, the same holds simultaneously for all $a \in \A$ with probability at least $1-2k\delta$.
\end{lemma}

Note that we renamed $b$ to $a$, and the second part follows by a union bound over the $k$ arms.

\subsection{Regret Analysis  (Known Corruption Budget)} \label{sec:regret_analysis}

We start by showing that with high probability, Algorithm~\ref{alg:cpe} never eliminates the optimal arm.  Recalling that the optimal arm is $a^* = \argmax_{a \in \A_0} \langle \theta, a \rangle$, we trivially have $a^* \in \A_0$.  We first show that $a^* \in \A_1$ with high probability. 

%

    At the end of epoch $0$, the estimate $\htheta_0$ is formed.  Letting $\hat{a} = \argmax_{a \in \A_0} \langle \htheta_0, a \rangle$ and conditioning on the second part of Lemma \ref{lem:robust_conf} holding true, we have
  \begin{align}
    \langle \htheta_0, \hat{a} - a^* \rangle &\leq \langle \htheta, \hat{a}\rangle - \langle \theta, \hat{a}\rangle + \langle \theta, a^*\rangle - \langle \htheta, a^*\rangle. \label{eq:ahat1} \\
    &= \langle \htheta - \theta, \hat{a}\rangle + \langle \theta -  \htheta, a^*\rangle \label{eq:ahat2} \\
    &\leq 2 \sqrt{\frac{4d}{m_0}\log\Big(\frac{1}{\delta}\Big)} + 
    \frac{2C}{m_0 \nu} \sqrt{4d(1+ \nu m_0)}, \label{eq:ahat3}
  \end{align}
where \eqref{eq:ahat1} holds since $a^*$ maximizes $\langle \theta, a\rangle$, while \eqref{eq:ahat3} is due to~\eqref{eq:mean_diff}.  

Notice that \eqref{eq:ahat3} is precisely the condition used in the algorithm to retain arms. It follows that the algorithm will not eliminate the optimal arm at the end of the first epoch, with probability at least $1-2k\delta$.  By applying an induction argument with the same steps as above in subsequent epochs, it follows that if $\tilde{H}$ is any almost-sure upper bound on the number of epochs $H$, then with probability at least $1 - 2k\tilde{H}\delta$, the algorithm will retain the optimal arm in every epoch.  We claim that we can set $\tilde{H} = \log_2 (T)$.  To see this, note that $m_h = 2^{h}m_0$, and because each epoch's length $u_h \geq m_h$ is hence greater than $2^h$, the total number of epochs $H$ is deterministically upper bounded by $\log_2 T$. 

In the remainder of the proof, we condition on the preceding events that hold with probability at least $1 - 2k\tilde{H}\delta$ (we will later rescale $\delta$ by $2k\tilde{H}$ for consistency with the statement of Theorem \ref{thm:regret_bound}).  Hence, the optimal arm is retained, and the confidence bounds \eqref{eq:mean_diff} apply in all epochs. 

We proceed by analyzing the regret.  Fix $h \in \lbrace 0, \dots, H-1 \rbrace$, and let $u_h(a)$ denote the number of times arm $a$ is played in epoch $h$. From the definition of regret, we have
  \begin{align}
    R_T &= \sum_{t=1}^T \big(  \langle \theta, a^* \rangle - \langle \theta, A_t \rangle \big) \label{eq:regret_bound} \\
        &= \sum_{h=0}^{H-1} \sum_{a \in \A_h} u_h(a) \big(\langle \theta, a^*\rangle - \langle\theta, a\rangle\big) \label{eq:regret_0}  \\
        &\leq 2u_0 + \sum_{h=1}^{H-1} \sum_{a \in \A_h} u_h(a) \big(\langle \theta, a^*\rangle - \langle\theta, a\rangle\big) \label{eq:regret_1a} \\ 
        &\leq 2u_0 + \sum_{h=1}^{H-1}  \sum_{a \in \A_h} u_h(a) \Big(4 \sqrt{\tfrac{4d}{m_{h-1}}\log\big(\tfrac{1}{\delta}\big)} + \tfrac{4C}{m_{h-1} \nu} \sqrt{4d(1+ \nu m_0)}\Big) \label{eq:regret_2a}\\
        &= 2u_0 + \sum_{h=1}^{H-1}  u_h \Big(4 \sqrt{\tfrac{4d}{m_{h-1}}\log\big(\tfrac{1}{\delta}\big)} + \tfrac{4C}{m_{h-1} \nu} \sqrt{4d(1+ \nu m_0)}\Big) 
  \end{align}
  \begin{align}  
        &\leq 2u_0 + \sum_{h=1}^{H-1}  2m_h(1+ \nu m_0) \Big(4 \sqrt{\tfrac{4d}{m_{h-1}}\log\big(\tfrac{1}{\delta}\big)} + \tfrac{4C}{m_{h-1} \nu} \sqrt{4d(1+ \nu m_0)}\Big) \label{eq:regret_3a} \\
        &= 2u_0 + \sum_{h=1}^{H-1}  4m_h \Big(4 \sqrt{\tfrac{4d}{m_{h-1}}\log\big(\tfrac{1}{\delta}\big)} + \tfrac{4Cm_0}{m_{h-1}} \sqrt{8d}\Big) \label{eq:regret_4a}\\
        &= 2u_0 + \sum_{h=1}^{H-1} \Big( 64 \sqrt{d m_{h-1}\log\big(\tfrac{1}{\delta}\big)} + 64Cm_0 \sqrt{2d}\Big)\label{eq:regret_5} \\
        &\leq 2u_0 +  64c_0 \sqrt{d T\log\big(\tfrac{1}{\delta}\big)} + 64Cm_0 \sqrt{2d}\log_2 T\label{eq:regret_final},
  \end{align}
  where:
  \begin{itemize}[leftmargin=5ex,itemsep=0ex,topsep=0.25ex]
    \item \eqref{eq:regret_0} uses the definition of $u_h(a)$;
    \item \eqref{eq:regret_1a} uses the fact that the instant regret is at most $2$ and the length of the first epoch is $u_0$;
    \item \eqref{eq:regret_2a} follows since
  \begin{align}
    \langle \theta, a^*\rangle - \langle\theta, a\rangle &\leq  \langle \htheta_{h-1}, a^*\rangle - \langle\htheta_{h-1}, a\rangle + 2 \Big(\sqrt{\tfrac{4d}{m_{h-1}}\log\big(\tfrac{1}{\delta}\big)} + \tfrac{C}{m_{h-1} \nu} \sqrt{4d(1+ \nu m_0)}\Big) \label{eq:substep1} \\
    &\leq 4 \Big(\sqrt{\tfrac{4d}{m_{h-1}}\log\big(\tfrac{1}{\delta}\big)} + \tfrac{C}{m_{h-1} \nu} \sqrt{4d(1+ \nu m_0)}\Big), \label{eq:substep2}
  \end{align}
  where \eqref{eq:substep1} follows by using \eqref{eq:mean_diff} to upper and lower bound $\langle \theta, a^*\rangle$ and $\langle\theta, a\rangle$ respectively, and \eqref{eq:substep2} follows from the condition \eqref{eq:retain_arms_condition} for retaining arms (with $\hat{C}_h = C$);
    \item \eqref{eq:regret_3a} follows from Lemma~\ref{lemma:epoch_length}; 
    \item \eqref{eq:regret_4a} follows by choosing $\nu = \frac{1}{m_0}$ as per Theorem \ref{thm:regret_bound};\footnote{The final term in~\eqref{eq:regret_3a} contains both increasing and decreasing factors with respect to $\nu$, thus not permitting us to set $\nu$ to an arbitrary small value.  The choice $\nu = \frac{1}{m_0}$ is convenient for the analysis, though we do not claim it to be optimal, nor necessarily the best in practice.}
    \item \eqref{eq:regret_5} follows by applying $m_h = 2m_{h-1}$ and simplifying;
    \item \eqref{eq:regret_final} holds for some constant $c_0 > 0$ since a sum of exponentially increasing terms is upper bounded by a constant times the last (and the longest epoch length is trivially at most $T$), and we also use $H \le \tilde{H}= \log_2 (T)$ for the second term.
  \end{itemize}  

   We note that the first term in \eqref{eq:regret_final} is insignificant compared to the last term, since $u_0 \leq 4m_0$ by Lemma~\ref{lemma:epoch_length} with $\nu=\frac{1}{m_0}$. 
  Since the preceding analysis holds with probability at least $1 - 2k\tilde{H}\delta$, we scale $\delta \leftarrow \frac{\delta}{2k\tilde{H}}$ to obtain that with probability at least $1-\delta$ that 
  \begin{equation} \label{eq:regret}
    R_T = \tilde{O}\Big(\sqrt{dT\log\big(\tfrac{k}{\delta}\big)} + Cd^{3/2}\log T\Big),
  \end{equation}
  where the $O(\log \tilde{H}) = O(\log \log T)$ term and the $O(\log \log d)$ term from $m_0$ are absorbed into the $\tilde{O}(\cdot)$ notation.

\subsection{Unknown Corruption Budget}


Recall that for unknown $C$, the algorithm uses $\hat{C}_h =  \min \lbrace \frac{\sqrt{T}}{m_0 \log_2 T}, m_0\sqrt{d} 2^{\tilde{H}-h} \rbrace$, where $\tilde{H} = \log_2 T$ is a deterministic upper bound on the number $H$ of epochs.  Recall also that we assume $C < \frac{\sqrt T}{m_0 \log_2 T}$, with the case $C \ge \frac{\sqrt T}{m_0 \log_2 T}$ discussed following Theorem \ref{thm:regret_bound2}.


A few observations are in order before we proceed:
\begin{itemize}[leftmargin=5ex,itemsep=0ex,topsep=0.25ex]
  \item We refer to epochs for which $\hat{C}_h \geq C$ as \emph{safe}, as the adversary cannot eliminate the optimal arm in these epochs. This follows from the analysis of Section \ref{sec:regret_analysis}, where we showed that if the true $C$ is used in the criterion for retaining arms, then the optimal arm is retained with high probability (thus, the same follows if $\hat{C}_h \geq C$ is used instead). 
  \item Let $h'$ be the first epoch index for which $C \geq \hat{C}_h$. It follows that 
  \begin{equation} \label{eq:remaining_epochs_size}
    C \geq 2^{\tilde{H}-h'}m_0\sqrt{d}, \quad \text{and hence} \quad \log_2 \big(\tfrac{C}{m_0\sqrt{d}}\big) \geq \tilde{H} - h'.
  \end{equation}
  Thus, the number of remaining epochs $H-h'$ is at most $\log_2 \big(\tfrac{C}{m_0\sqrt{d} }\big)$.
  \item In the worst case, the adversary can distribute its budget $C$ among these later epochs with indices $h \geq h'$ to force arms to be eliminated.  That is, with $C_h$ denoting the budget used in epoch $h$ (so that $\sum_{h=0}^{H-1} C_h = C$), it is possible to have $C_h > \hat{C}_h$ in these epochs. We therefore refer to these epochs as \emph{unsafe}.  Note that since $C < \frac{\sqrt T}{m_0 \log_2 T}$, the adversary does not have enough budget to make any ``early'' epoch for which $\min\lbrace \frac{\sqrt{T}}{m_0 \log_2 T}, m_0\sqrt{d}2^{\tilde{H}-h} \rbrace = \frac{\sqrt{T}}{m_0 \log_2 T}$ to be unsafe.
\end{itemize} 

We consider the first unsafe epoch $h'$, and suppose that the adversary eliminates the optimal arm. Hence, it holds that $C_{h'} > \hat{C}_{h'}$. Our goal will be to show that although the optimal arm gets eliminated, an arm $a_{h'}$ that is "almost" as good as the optimal arm is retained.

Because $a^*$ got eliminated, the rule \eqref{eq:retain_arms_condition} for retaining arms implies:
\begin{equation}\label{eq:retaining_condition_not_met}
  \max_{a \in \A_{h'}} \langle \hat{\theta}_{h'}, a- a^* \rangle > 2 \sqrt{\tfrac{4d}{m_{h'}}\log\big(\tfrac{1}{\delta}\big)} +  
  \tfrac{2\hat{C}_{h'}}{m_{h'} \nu} \sqrt{4d(1 + \nu m_0)}.
\end{equation}
Let $a_{h'} = \argmax_{a \in \A_{h'}} \langle \hat{\theta}_{h'}, a- a^* \rangle$, and observe that $a_{h'}$ is not eliminated, i.e., $a_{h'} \in \A_{h'+1}$.

We again condition on the second part of Lemma \ref{lem:robust_conf}, which holds with probability at least $1-2k\tilde{H}\delta$. This event implies for every $a \in \A_{h'}$ that
\begin{equation}
  |\langle a, \htheta_{h'} - \theta \rangle| \leq \sqrt{\tfrac{4d}{m_{h'}}\log(\tfrac{1}{\delta})} + 
  \tfrac{C}{m_{h'} \nu} \sqrt{4d(1+ \nu m_0)},
\end{equation}
and note that this holds for both $a_{h'}$ and $a^*$, since both $a_{h'},a^* \in \A_{h'}$ (recall that the adversary does not have enough budget to remove $a^*$ from $\A_{h'}$ according to the definition of $h'$). Combining the two associated bounds, we obtain
\begin{equation}\label{eq:ucb}
  \langle  \htheta_{h'}, a_{h'} - a^*\rangle \leq \langle  \theta, a_{h'} - a^*\rangle + 2\sqrt{\tfrac{4d}{m_{h'}}\log(\tfrac{1}{\delta})} + 
  \tfrac{2C}{m_{h'} \nu} \sqrt{4d(1+ \nu m_0)},
\end{equation}
and combining \eqref{eq:retaining_condition_not_met} with \eqref{eq:ucb} gives
\begin{align}
  \langle  \theta, a_{h'}\rangle &> \langle  \theta, a^*\rangle  - \tfrac{2}{m_{h'} \nu} \sqrt{4d(1+ \nu m_0)} (C - \hat{C}_{h'}) \\
  &\geq \langle  \theta, a^*\rangle  - \tfrac{2C}{m_{h'} \nu} \sqrt{4d(1+ \nu m_0)} \label{eq:thet_ahprime}.
\end{align}
By denoting $\eta_{h'}:= \langle  \theta, a^*\rangle - \langle  \theta, a_{h'}\rangle$, we can rewrite ~\eqref{eq:thet_ahprime} as 
\begin{equation} \label{eq:eta_h_prime}
  \eta_{h'} < \tfrac{2C}{m_{h'} \nu} \sqrt{4d(1+ \nu m_0)}.
\end{equation}
Note that $\eta_{h'}$ represents the additional regret (due to the optimal arm elimination) that our algorithm can incur in epoch $h' + 1$ for each arm pull.

In epoch $h'+1$, the optimal arm $a^*$ is already eliminated in the worst case, and again, the adversary can potentially eliminate the best remaining arm $a^*_{h'+1} = \argmax_{a \in \A_{h'+1}} \langle  \theta, a \rangle$ by using $C_{h'+1} > \hat{C}_{h'+1}$. By repeating the same arguments as those leading to~\eqref{eq:eta_h_prime}, the additional regret can be written as $\eta_{h' + 1} = \langle  \theta, a^*\rangle - \langle  \theta, a_{h'+1}\rangle$, where $a_{h'+1} = \argmax_{a \in \A_{h'+1}} \langle \hat{\theta}_{h'+1}, a- a_{h'+1}^* \rangle$, and it holds that
\begin{align}   
    \eta_{h' + 1} &< \eta_{h'} + \tfrac{2C}{\nu m_{h'+1}} \sqrt{4d(1+ \nu m_0)} \\
    &<s \tfrac{2C}{\nu} \sqrt{4d(1+ \nu m_0)} \Big(\tfrac{1}{m_{h'}} + \tfrac{1}{m_{h'+1}}\Big).
\end{align}
By induction, for each of the remaining epochs, we have:
\begin{align}
  \eta_{h' + l} &< \tfrac{2C}{\nu} \sqrt{4d(1+ \nu m_0)} \sum_{i=h'}^{h'+l} \tfrac{1}{m_{i}} \\
                &\leq \tfrac{2c_0 C}{m_{h'} \nu} \sqrt{4d(1+ \nu m_0)}, \label{eq:induction2}
\end{align}
where \eqref{eq:induction2} holds for some constant $c_0$, since the sum of exponentially shrinking terms is dominated by the first one in the sum. Next, by substituting $\nu = \frac{1}{m_0}$ and using $u_h \leq 4m_h$, we deduce that the total additional regret that we incur is:
\begin{align} \label{eq:last_term_bound}
  \sum_{h=h'}^{H-2}  u_{h+1} \eta_{h} &\leq \sum_{h = h'}^{H-2} 4m_{h+1} \eta_{h} \\
  &\le \tfrac{8c_0 m_0C}{m_{h'}} \sqrt{8d} \sum_{h = h'}^{H-2} m_{h+1} \label{eq:hsum2}  \\
  &\leq 8c'_1 m_0  C \sqrt{8d}\tfrac{m_{H-2}}{m_{h'}} \label{eq:hsum3} \\
  &= O\big(C^2\big) \label{eq:regret_1b},
\end{align}
where \eqref{eq:hsum2} applies \eqref{eq:induction2} with $\nu = \frac{1}{m_0}$, \eqref{eq:hsum3} holds for some constant $c'_1$ (depending on $c_0$) since a sum of exponentially increasing terms is upper bounded by a constant times the last term, and \eqref{eq:regret_1b} uses
\begin{equation}
  \tfrac{m_{H-2}}{m_{h'}} = \tfrac{2^{H-2} m_0}{2^{h'}m_0} = 2^{H-h'-2} \leq 2^{\tilde{H}-h'} \leq 2^{\log_2 \big(\tfrac{C}{m_0\sqrt{d}}\big)} = \frac{C}{m_0\sqrt{d}},
\end{equation}
with the first inequality using \eqref{eq:remaining_epochs_size}.

We now proceed to bound the regret similarly to \eqref{eq:regret_bound}:
  \begin{align}
    R_T &= \sum_{t=1}^T \big(\langle\theta, a^* \rangle - \langle \theta, A_t \rangle\big) \\
        &= \sum_{h=0}^{H-1} \sum_{a \in \A_h} u_h(a) \big(\langle \theta, a^*\rangle - \langle\theta, a\rangle\big) \\
    &\leq 2u_0 + \sum_{h=1}^{H-1} \sum_{a \in \A_h} u_h(a) \big(\langle \theta, a^*\rangle - \langle\theta, a\rangle\big) \label{eq:regret_1} \\ 
    &\leq 2u_0 + \Big(\sum_{h=1}^{H-1}  4m_h \Big(4 \sqrt{\tfrac{4d}{m_{h-1}}\log\big(\tfrac{1}{\delta}\big)} 
    + \tfrac{2\hat{C}_{h-1} m_0}{m_{h-1}} \sqrt{8d} + \tfrac{2C m_0}{m_{h-1}} \sqrt{8d}\Big)\Big) + \sum_{h=h'}^{H-2}  u_{h+1} \eta_{h}, \label{eq:separate_terms}
  \end{align}
where the $\hat{C}_{h-1}$ term comes from the condition \eqref{eq:retain_arms_condition} for retaining arms, and the $C$ term comes from the use of \eqref{eq:mean_diff} (this is in contrast to the known $C$ case, in which the former term also uses $C$).

It remains to bound the terms in \eqref{eq:separate_terms} separately. We have already shown the bound on the last term (see in \eqref{eq:regret_1b}). Next, we show:
\begin{align}
  \sum_{h=1}^{H-1}  4m_h \tfrac{2\hat{C}_{h-1} m_0}{m_{h-1}} \sqrt{8d} &= 16m_0 \sqrt{8d}\sum_{h=1}^{H-1} \hat{C}_{h-1} \label{eq:hat_step1} \\ 
  &= 16m_0 \sqrt{8d}\sum_{h=1}^H \min \lbrace \tfrac{\sqrt{T}}{m_0\log_2 T}, m_0\sqrt{d}2^{\tilde{H}-h}  \rbrace \label{eq:hat_step2} \\
  &\leq 16m_0 \frac{\sqrt{8dT}}{m_0}  \label{eq:hat_step3} \\
  &= O(\sqrt{dT})\label{eq:regret_2},
\end{align}
where \eqref{eq:hat_step1} uses $m_h = 2m_{h-1}$, \eqref{eq:hat_step2} substitutes the choice of $\hat{C}_h$, and \eqref{eq:hat_step3} upper bounds the minimum by the first term and applies $H \le \tilde{H} = \log_2 T$.

Similarly, recalling that $m_0 = 4d(\log\log d + 18)$ and $m_h = 2m_{h-1}$, we have
\begin{align}
  \sum_{h=1}^{H-1}  4m_h \tfrac{2 C m_0}{m_{h-1}} \sqrt{8d} &= 16m_0 \sqrt{8d}C \log T \\
  &= \tilde{O}(d^{3/2} C\log T) \label{eq:regret_3},
\end{align}
and 
\begin{align}
  \sum_{h=1}^{H-1}  4m_h \Big(4 \sqrt{\tfrac{4d}{m_{h-1}}\log\big(\tfrac{1}{\delta}\big)} \Big) &\leq 64c_0 \sqrt{d T\log\big(\tfrac{1}{\delta}\big)} \label{eq:c_0_appear} \\
  &= O\Big(\sqrt{dT \log\big(\tfrac{1}{\delta}\big)}\Big) \label{eq:regret_4},
\end{align}
where \eqref{eq:c_0_appear} holds for some constant $c_0 > 0$ similarly to \eqref{eq:regret_final}.  Then, again replacing $\delta \leftarrow \frac{\delta}{2k\tilde{H}}$ in the same way as the known $C$ case, the term \eqref{eq:regret_4} becomes
\begin{equation}
    \tilde{O}\Big(\sqrt{dT \log\big(\tfrac{k}{\delta}\big)}\Big), \label{eq:regret_4b}
\end{equation}
and the associated probability is now $1-\delta$.

Combining \eqref{eq:regret_1b}, \eqref{eq:regret_2}, \eqref{eq:regret_3} and \eqref{eq:regret_4b}, we arrive at the regret bound:
\begin{equation}
  R_T = \tilde{O}\bigg(\sqrt{dT \log\big(\tfrac{k}{\delta}\big)} + d^{3/2} C\log T + C^2\bigg).
\end{equation}

\section{Proofs for Section \ref{sec:contextual} (Contextual Greedy Algorithm)} \label{sec:pf_context}

For reference, a complete description of the greedy algorithm is given in Algorithm \ref{alg:cg}.

\begin{algorithm}[t]
    \caption{Contextual Greedy}
    \label{alg:cg}
    \begin{algorithmic}[1]
        \Require Initialize $\htheta_1$ arbitrarily
        \For {$t = 1,2,\dotsc, T$}
          \State Receive a set of contexts $\lbrace a_{1,t}, \dots, a_{k,t} \rbrace$ 
          \State Choose arm $I_t = \argmax_{i \in \lbrace 1, \dots, K \rbrace} \langle \hat{\theta}_t, a_{i,t} \rangle$
          \State Observe: $Y_t = \langle \theta, a_{I_t,t} \rangle + \epsilon_t + c_t(a_{I_t,t})$
          \State Update $\htheta_{t+1} \in \argmin_{\theta'} \sum_{\tau=1}^{t} (\langle \theta', a_{I_{\tau},\tau} \rangle - Y_\tau)^2$
          \Comment{break ties arbitrarily}
        \EndFor
        \State \textbf{end for}
    \end{algorithmic}
\end{algorithm} 

Before proving Theorem~\ref{thm:greedy_regret_bound_thm}, we introduce some useful auxiliary results. Our proof builds heavily on that of \cite{kannan_smoothed_2018}, whose setup matches ours but does not consider adversarial attacks (i.e., their setup corresponds to the case that $C = 0$).

\begin{lemma}[Lemma 3.1 \cite{kannan_smoothed_2018}]
    If $\|a_{i,t}\|_2 \leq 1$ for all $i,t$, then for any $t_0 < T$, we have
    \begin{equation} \label{eq:kannan_1}
      R_T \leq 2t_{0} + 2 \sum_{t=t_0}^{T} \| \theta - \htheta_t \|_2. 
    \end{equation}
\end{lemma}

\begin{proof}
  We reproduce the proof for the sake of demonstrating the use of the greedy rule in \eqref{eq:greedy}. Recall that the least-squares estimator $\htheta_t$ is computed by using the previously observed attacked rewards.
  We can bound the regret incurred in the first $t_0$ rounds by the maximum regret value $2$. Then, we consider the regret $r_t$ incurred at time $t$; denoting $i_t^* = \argmax_{i \in \lbrace 1, \dots, K \rbrace} \langle \theta, a_{i,t} \rangle$, we have
  \begin{align}
  r_t &= \langle \theta, a_{i_t^*,t} \rangle - \langle \theta, a_{I_t,t} \rangle \\
      &= \big(\la \theta, a_{i_t^*,t} \ra - \la \htheta_t, a_{i_t^*,t} \ra\big) - \big(\la \theta, a_{I_t,t} \ra - \la \htheta_t, a_{I_t,t} \ra \big) + \big(\la  \htheta_t, a_{i_t^*,t}\ra - \la \htheta_t, a_{I_t,t} \ra \big) \\
      &\leq \big(\la \theta, a_{i_t^*,t} \ra - \la \htheta_t, a_{i_t^*,t} \ra\big) - \big(\la \theta, a_{I_t,t} \ra - \la \htheta_t, a_{I_t,t} \ra \big) \label{eq:greedy_rule_used}\\
      &\leq \big | \la \theta, a_{i_t^*,t} \ra - \la \htheta_t, a_{i_t^*,t} \ra\big | + \big|\la \theta, a_{I_t,t} \ra - \la \htheta_t, a_{I_t,t} \ra \big| \\
      &\leq \| \theta - \htheta_t \|_2 \| a_{i_t^*,t} \|_2 + \| \theta - \htheta_t \|_2 \| a_{I_t,t} \|_2 \\
      &\leq 2 \| \theta - \htheta_t \|_2,
\end{align}
where \eqref{eq:greedy_rule_used} follows since $I_t$ is selected greedily, and hence
$\la  \htheta_t, a_{i_t^*,t}\ra - \la \htheta_t, a_{I_t,t} \ra \leq 0$. 
\end{proof}

\begin{lemma} \label{lemma:estimator_confidence_greedy}
    For each round $t$, let $\Gamma_t = \sum_{\tau \leq t} a_{I_\tau}a_{I_\tau}^T$, and suppose that all contexts satisfy $\| a_{i,t}\|_2 \leq 1$, the reward noise is $1$-sub-Gaussian, and the attack budget is $C\geq0$. If $\lambda_{\min}(\Gamma_t) > 0$, then with probability at least $1-\delta$, it holds that
    \begin{equation} \label{eq:kannan_2}
      \| \theta - \htheta_t \|_2 \leq \frac{\sqrt{2dt \log(td/\delta)}}{\lambda_{\min}(\Gamma_t)} + \frac{C}{\lambda_{\min}(\Gamma_t)}.
    \end{equation}  
\end{lemma}

\begin{proof}
Since $\lambda_{\min}(\Gamma_t) > 0$, the matrix $\Gamma_t$ is invertible, and we can use the standard closed-form least squares solution expression: $\htheta_t = \Gamma_t^{-1} \sum_{\tau \leq t} a_{I_\tau, \tau} Y_\tau$.  Decomposing $Y_\tau$ into the sum of the reward, noise, and adversarial corruption (similarly to \eqref{eq:considered4}), we obtain
\begin{equation}
  \htheta_t = \theta + \Gamma_t^{-1} \sum_{\tau \leq t} a_{I_\tau, \tau} \epsilon_{\tau} + 
  \Gamma_t^{-1} \sum_{\tau \leq t} a_{I_\tau, \tau} c_{\tau}(a_{I_\tau,\tau}),
\end{equation}
which implies that
\begin{align}
  \| \htheta_t - \theta \|_2 &\leq \big \| \Gamma_t^{-1} \sum_{\tau \leq t} a_{I_\tau, \tau} \epsilon_{\tau}\big \|_2 + \big\| \Gamma_t^{-1} \sum_{\tau \leq t} a_{I_\tau, \tau} c_{\tau}(a_{I_\tau,\tau})\big\|_2 \\ 
  &\leq \frac{1}{\lambda_{\min}(\Gamma_t)} \Big( \big \|\sum_{\tau \leq t} a_{I_\tau, \tau} \epsilon_{\tau}\big \|_2 + \big\| \sum_{\tau \leq t} a_{I_\tau, \tau} c_{\tau}(a_{I_\tau,\tau})\big\|_2\Big). \label{eq:greedy_decomp}
\end{align}
With probability at least $1 - \delta$, the first term is bounded as \cite[Lemma A.1]{kannan_smoothed_2018}
\begin{equation} \label{eq:greedy_noise}
  \big \|\sum_{\tau \leq t} a_{I_\tau, \tau} \epsilon_{\tau}\big \|_2 \leq  \sqrt{2dt \log(td/\delta)}.
\end{equation}
For the second term in \eqref{eq:greedy_decomp}, we note the following:
\begin{align}
  \big \|\sum_{\tau \leq t} a_{I_\tau, \tau} c_{\tau}(a_{I_\tau,\tau})\big \|_2 &\leq \sum_{\tau \leq t} \big \| a_{I_\tau, \tau} c_{\tau}(a_{I_\tau,\tau})\big \|_2 \\
  &\leq \sum_{\tau \leq t}  |c_{\tau}(a_{I_\tau,\tau})| \cdot  \big\| a_{I_\tau, \tau} \big \|_2 \\
  &\leq \sum_{\tau \leq t} |c_{\tau}(a_{I_\tau,\tau})| \\
  & \leq C. \label{eq:greedy_corruption}
\end{align}
Combining \eqref{eq:greedy_decomp} with \eqref{eq:greedy_noise} and \eqref{eq:greedy_corruption} completes the proof.
\end{proof}

We are now ready to prove Theorem~\ref{thm:greedy_regret_bound_thm}.
\begin{proof}[Proof of Theorem \ref{thm:greedy_regret_bound_thm}]
  We follow the steps of the proof of \cite[Thm.~3.1]{kannan_smoothed_2018}.  We start by proving the following counterpart of \cite[Corollary 3.1]{kannan_smoothed_2018}: Letting $t_0 = \max \big \lbrace  4 +  2 \sqrt{\tfrac{1}{2} \log \big( \tfrac{k}{\delta}\big)}, 32 \log(\tfrac{4T}{\delta}), \frac{80\log(2dT / \delta)}{\lambda_0}\big\rbrace$,
  for every $t \geq t_0$, it holds with probability at least $1 - \delta$ that
  \begin{equation} \label{eq:kannan_2a}
    \| \theta - \htheta_t \|_2 \leq \frac{32 \sqrt{d \log(2Td/\delta)}}{\lambda_0 \sqrt{t}} + \frac{16C}{\lambda_{0}t}.
  \end{equation} 
   To prove this, we use Lemma \ref{lemma:estimator_confidence_greedy} with $\frac{\delta}{2}$ in place of $\delta$; \eqref{eq:kannan_2a} will then follow once we show that 
  \begin{equation}
    \lambda_{\min}(\Gamma_t) \geq \frac{t \lambda_0}{16}. 
  \end{equation}
    This result is shown in \cite[Lemma B.1]{kannan_smoothed_2018} (making use of the assumption $t \ge t_0$), and only requires that the random context perturbations are $(r, \lambda_0 )$-diverse.  Thus, it continues to hold in the corrupted setting with $C > 0$.


  Combining \eqref{eq:kannan_1} and \eqref{eq:kannan_2a}, we have with probability at least $1 - \delta$ that
  \begin{align}
    R_T &\leq 2t_{0} + 2 \sum_{t=t_0}^{T} \Big(  \tfrac{32 \sqrt{d \log\frac{2Td}{\delta}}}{\lambda_0 \sqrt{t}} + \tfrac{16C}{\lambda_{0}t}\Big) \\
    &\leq 2t_{0} + \tfrac{128\sqrt{dT \log \big(\tfrac{2Td}{\delta}\big)}}{\lambda_{0}} + \tfrac{64C \log T}{\lambda_0},
  \end{align}
  since $\sum_{t=1}^T \frac{1}{\sqrt t} \le 2\sqrt{T}$ and $\sum_{t=1}^T \frac{1}{t} \le 2\log{T}$ (with the latter assuming $T > 2$).  Substituting the definition of $t_0$, it follows that
    with probability at least $1 - \delta$, we have
  \begin{equation}
    R_T = O \bigg(\tfrac{1}{\lambda_0} \Big(\sqrt{dT \log \big(\tfrac{Td}{\delta}\big)} + C \log T + \log\big(\tfrac{dT}{\delta}\big)\Big) + \sqrt{\log(\tfrac{k}{\delta})} \bigg).
  \end{equation}
\end{proof}

We now consider the special case of Gaussian perturbations, i.e., each $\xi_{i,t}$ is drawn independently from $\mathcal{N}(0,\eta^2I)$ for some $\eta > 0$. We make use of the above results, as well as the ones from \cite[Section 3.2]{kannan_smoothed_2018}, to show that Algorithm~\ref{alg:cg} has sublinear regret under small perturbations. This is formally stated in the following corollary.

\begin{corollary} \label{corr:greedy}
Assume the context perturbations $\xi_{i,t}$ are drawn independently from $\mathcal{N}(0,\eta^2 I)$ for all $i,t$, the reward noise is $1$-sub-Gaussian, and the attack budget of the adversary is $C \geq 0$. Then for a fixed number of arms $k$ and $\eta \leq O((\sqrt{d \log(Tkd/\delta)})^{-1})$, with probability at least $1 - \delta$, the greedy algorithm (Algorithm~\ref{alg:cg}) has regret bounded by
  \begin{equation}
    R_T = O \Big(\tfrac{\sqrt{Td}}{\eta^2} \log\big(\tfrac{dT}{\delta}\big)^{3/2} + \tfrac{C(\log T)^2}{\eta^2} \Big).
  \end{equation}
\end{corollary}

\begin{proof}
  The proof strategy is to invoke Theorem~\ref{thm:greedy_regret_bound_thm}, and in particular, since $k$ is assumed to be fixed, its variant stated in \eqref{eq:nicer_version}. 
  
    In Theorem~\ref{thm:greedy_regret_bound_thm}, it is assumed that $\|a_{i,t} \|_2 \leq 1$. However, the right-hand side of one here was only chosen for convenience, and (as noted in \cite{kannan_smoothed_2018}), the same result holds when $\|a_{i,t} \|_2 \leq R'$ for any given $R' \ge 1$ satisfying $R' = O(1)$.  This change only affects the constants, in particular introducing an $(R')^{3/2}$ term \cite{kannan_smoothed_2018}; this still behaves as $O(1)$ since we focus on the case that $R' = O(1)$.

    To invoke such a variant of Theorem~\ref{thm:greedy_regret_bound_thm}, we need to condition on an event that ensures $\|a_{i,t} \|_2 \leq R'$ for some constant $R'>0$ for every $i,t$ 
    and the context perturbations need to be ($r$, $1/T$)--bounded and $(r,\lambda_0)$--diverse for some $\lambda_0 > 0$ and $r \leq R'$. Next, we show these conditions hold for the case of Gaussian context perturbations.

  Towards that end, we start by outlining some results from \cite{kannan_smoothed_2018}. First \cite[Lemma 3.5]{kannan_smoothed_2018} states that when $\hat{R} \geq \eta \sqrt{ 2 \log(2kdT/\delta)}$, we have that:
  \begin{equation} \label{eq:leading_to_truncated_g}
    \mathbb{P}\big[|\xi_{i,t}[j]| \leq \hat{R}, \; 
    \forall i \in \lbrace 1, \dots, K \rbrace, t \leq T, j \in \lbrace 1, \dots, d \rbrace \big] \geq 1 - \delta/2.
  \end{equation}
  
  In what follows, we set $\hat{R}=2\eta \sqrt{ 2 \log(2kdT/\delta)}$, and
  we condition on the event in \eqref{eq:bounded_g} holding true.
  Then, \cite[Lemma 3.6]{kannan_smoothed_2018} states that when $\|\mu_{i,t} \|_2 \leq 1$, we have 
  \begin{equation} \label{eq:bounded_g}
    \|a_{i,t} \|_2 \leq 1 + \sqrt{d}\hat{R}:= R'\quad \text{ for all $i,t$},
  \end{equation}
  and the perturbations are $(r,1/T)-\text{bounded for } r \geq \eta \sqrt{2 \log T}$.  

  The perturbation distribution conditioned on the event in~\eqref{eq:leading_to_truncated_g} holding true is a truncated Gaussian supported on $[-\hat{R}, \hat{R}]$, and satisfies $(r, \lambda_0)$-diversity when \cite{kannan_smoothed_2018}
  \begin{equation} \label{eq:diverse_g}
    \lambda_0 = \Omega\Big(\frac{\eta^4}{r^2}\Big) = \Omega\Big(\frac{\eta^2}{\log T}\Big).
  \end{equation} 
  
   Finally, from the definitions of $R'$ in \eqref{eq:bounded_g} and $\hat{R}=2\eta \sqrt{ 2 \log(2kdT/\delta)}$, we have 
  \begin{equation}
    R' = 1 + \sqrt{d}\hat{R} \leq  2 \max \lbrace 1, \sqrt{d}\hat{R} \rbrace
    = 2 \max\lbrace 1, 2 \eta\sqrt{2d\log(dkT/ \delta)} \rbrace,
  \end{equation}
  and we see that to have $R' = O(1)$, it suffices to have $\eta \leq O((\sqrt{d \log(Tkd/\delta)})^{-1})$. 
  
  Finally, we apply the above mentioned variant of Theorem~\ref{thm:greedy_regret_bound_thm} with parameter $R' = O(1)$.     
  By the union bound, the events in Theorem~\ref{thm:greedy_regret_bound_thm} and event~\eqref{eq:bounded_g} simultaneously hold with probability at least $1- \delta$. By substituting the bound on $\lambda_0$ from \eqref{eq:diverse_g} into \eqref{eq:nicer_version}, we arrive at
  \begin{equation}
    R_T = O \Big(\tfrac{\sqrt{Td}}{\eta^2} \log\big(\tfrac{dT}{\delta}\big)^{3/2} + \tfrac{C(\log T)^2}{\eta^2} \Big).
  \end{equation}
\end{proof}

\section{Proofs of Lower Bounds} \label{sec:pf_lower}

We prove Theorems \ref{thm:lower_d} and \ref{thm:lower_C} in Sections \ref{sec:conv_knownC_2} and \ref{sec:conv_unknownC} respectively.  In Sections \ref{sec:conv_knownC_1} and \ref{sec:conv_diverse}, we prove two $\Omega(C)$ lower bounds using standard arguments, e.g., see \cite{lykouris2018stochastic}.

%

\subsection{Lower Bound for $d = k = 2$ (Unknown $C$)} \label{sec:conv_unknownC}

Here we show that for $d=2$ dimensions and $k=2$ arms, for any algorithm that guarantees $R_T \le \bar{R}^{(0)}_T$ (say, with probability $1 - \delta$) for some uncorrupted regret bound $\bar{R}^{(0)}_T \le \frac{T}{16}$ when $C = 0$, there exists an instance in which $R_T = \Omega(T)$ (again with probability $1-\delta$) when the attack budget is $C = 2 \bar{R}^{(0)}_T$.  We show that this is true even with no noise, i.e., $\epsilon_t = 0$ for all $t$.

To prove this, consider an instance with feature vectors $a_1 = \big[\frac{1}{2},0\big]^T$ and $a_2 = \big[0,\frac{1}{4}\big]^T$, and parameter vector $\theta = \big[\frac{1}{2},\frac{1}{2}\big]^T$.  In this case, pulling $a_1$ incurs zero regret, and pulling $a_2$ incurs regret $\frac{1}{8}$.  Hence, by the assumption $R_T \le \bar{R}^{(0)}_T$, we see that $a_2$ is pulled at most $8 \bar{R}^{(0)}_T$ times when $C = 0$.

Now consider a different instance, with feature vectors $a_1 = \big[\frac{1}{2},0\big]^T$ and $a_2 = \big[0,\frac{3}{4}\big]^T$, and again $\theta = \big[\frac{1}{2},\frac{1}{2}\big]^T$.  In this case, $a_2$ incurs zero regret, and $a_1$ incurs regret $\frac{1}{8}$.  Suppose that $C = 2\bar{R}^{(0)}_T$, and consider an adversary that pushes the reward of $a_2$ from $\frac{3}{8}$ down to $\frac{1}{8}$ whenever it is pulled, at a cost of $c_t(A_t) = \frac{1}{4}$.  Since $C = 2\bar{R}^{(0)}_T$, the adversary can afford to do this $8\bar{R}^{(0)}_T$ times.

However, as long as the adversary is corrupting, the observed rewards are exactly the same as in the first instance above, in which we established that $a_2$ is pulled at most $8 \bar{R}^{(0)}_T$ times.  Since we assume that $\bar{R}^{(0)}_T \le \frac{T}{16}$, it follows that $a_1$ is pulled at least $T - 8 \bar{R}^{(0)}_T \ge \frac{T}{2}$ times, leading to $\Omega(T)$ regret.

\subsection{Lower Bound for $d=1$ and $k=2$ (Known $C$)} \label{sec:conv_knownC_1}

In the case that $C$ is known, we can obtain an $\Omega(C)$ lower bound using a simple argument from \cite[Sec.~5]{lykouris2018stochastic}, which we reproduce here for completeness.  Consider an instance with feature scalars $a_1 = 1$ and $a_2 = -1$.  Clearly, arm 1 is better when $\theta = 1$, but arm 2 is better when $\theta = -1$, and in both cases, the worse arm incurs regret $2$.
Consider the case that there is no random noise, and suppose that the adversary shifts every reward to zero until its budget is depleted, i.e., for $\lfloor C \rfloor$ rounds.  During these rounds, the learner must pull some arm at least $\lfloor C \rfloor / 2$ times, and for one of the two values of $\theta \in \{-1,1\}$, a cumulative regret of at least $\lfloor C \rfloor$ is incurred.  Since the two $\theta$ values are indistinguishable during these rounds, we conclude that $\Omega(C)$ regret is unavoidable.

\subsection{Lower Bound for $d = k > 2$ (Known $C$)} \label{sec:conv_knownC_2}

Here we generalize the argument of the previous subsection to deduce a stronger lower bound with a joint dependence on $d$ and $C$.  We consider the case that $d = k$, with $a_i$ being the $i$-th standard basis vector.  Hence, for any $\theta \in \RR^d$, we have $\langle a_i, \theta\rangle = \theta_i$.  We again consider the noiseless setting.

Consider $d$ different bandit instances, the $i$-th of which has $\theta = a_i$.  Hence, in the $i$-th instance, arm $i$ has reward $1$, and the rest have reward zero.  In addition, consider an adversary that pushes the reward of the $i$-th arm down to zero whenever it is pulled; this can again be done $\lfloor C \rfloor$ times.  Roughly speaking, the learner can do no better than pull each arm in succession, incurring $\Omega(Cd)$ regret.

To make this more precise, note that in the $i$-th instance, the adversary only runs out of its budget after the $i$-th arm is pulled $\lfloor C \rfloor$ times.  However, after $\big\lfloor \frac{\lfloor C\rfloor d}{2} \big\rfloor $ rounds, there must remain at least $\frac{d}{2}$ arms that have not been pulled $\lfloor C \rfloor$ times.  When $\theta = a_i$ for any $i$ corresponding to one of these $\frac{d}{2}$ arms, the regret incurred is $\Omega(Cd)$.


\subsection{Lower Bound for Diverse Contexts (Known $C$)} \label{sec:conv_diverse}

Finally, we argue that the approach of Section \ref{sec:conv_knownC_1} gives an $\Omega(C)$ lower bound on $R_T$ even under the assumption of diverse contexts.  Recall that in \eqref{eq:diverse} we consider fixed center points $\mu_1,\dotsc,\mu_k$ and assume that these are perturbed by $\xi \sim D$.  The following argument holds under any such setup satisfying the mild assumption that, in each round, a constant positive fraction (e.g., $0.01$) of the arms have regret lower bounded by some positive constant (e.g., $0.01$).

We again consider an adversary that pushes the reward to zero (while leaving the random reward noise unchanged) until the budget is exhausted.  Hence, for the first $\lfloor C \rfloor$ rounds, the learner learns nothing about $\theta$. The preceding assumption rules out pathological cases such as all arms being identical, and we conclude that constant regret is incurred per round (with constant probability), for $\Omega(C)$ regret total.\looseness=-1

\section{Attack Methods} \label{sec:attacks}

Recall that we consider the model $Y_t =  \langle \theta, A_t \rangle + \epsilon_t + c_t(A_t)$, where $\epsilon_t$ is random noise and $c_t(A_t)$ is the adversarial corruption.  In our experiments, we consider four attacks, summarized as follows

{\bf Garcelon {\em et al.}~attack.} In an attack proposed in \cite{garcelon_adversarial_2020}, the attacker selects a {\em target arm} $a_{\rm target} \in \A$ that it wants to trick the learner into thinking is optimal, and operates as follows: (i) If $a_{\rm target}$ is pulled, leave the reward unchanged; (ii) If any other arm is pulled, change the reward so that $Y_t = \tilde{\epsilon}_t$, with $\tilde{\epsilon}_t$ being artificial random noise generated by the adversary.

Since our adversary is assumed to know $\langle \theta, A_t \rangle$ and $\epsilon_t$ individually, we can alter this attack to remove the need for artificial noise; instead, in case (ii) above, the adversary shifts $\langle \theta, A_t \rangle$ down to zero, while leaving $\epsilon_t$ unchanged.  However, if $a_{\rm target}$ has a negative reward, then this attack will not make it appear optimal, so we also allow the adversary to shift down to a more generic value $v_{\rm target}$. 
In our experiments, we set $v_{\rm target} = -1$, which is the smallest possible reward.

Of course, the adversary will eventually run out of budget eventually, at which point the attack stops.  The same applies to all of the alternative attacks below.

{\bf Oracle MAB attack.} Since the adversary has full knowledge of the instance, we can consider the oracle attack proposed in \cite{jun_adversarial_nodate}, in which for some target arm $a_{\rm target}$, the following is performed: (i) If $a_{\rm target}$ is pulled, leave the reward unchanged; (ii) If any other arm $a$ is pulled, shift the reward down by $\max\{0, \langle \theta, a \rangle - \langle \theta, a_{\rm target} + \epsilon_0 \rangle\}$ for some $\epsilon_0 > 0$.  This means that every other arm looks $\epsilon_0$-suboptimal compared to $a_{\rm target}$.

{\bf Simple $\theta$-based attack.} In the contextual setting, fixing a target arm $a_{\rm target}$ by index (e.g., the first) may not be the most suitable choice, since the contexts are changing every round.  In the setup of Section \ref{sec:contextual}, we are primarily interested in the case that the perturbations are small, which mitigates this issue.  Nevertheless, when the perturbation variance $\eta$ becomes large enough, it is likely more effective for the attack to use a different strategy to choose $a_{\rm target}$.

Thus, we propose an attack that tries to make the arm most aligned with some vector $\theta_{\rm target}$ to appear best.  To do this, we simply the above variation of the Garcelon {\em et al.}~attack, but instead of letting $a_{\rm target}$ correspond to a fixed arm index (e.g., the first arm), the attacker updates $a_{\rm target}$ every round, choosing $a_{\rm target} = \argmax_{a \in \A_t} \langle \theta_{\rm target}, a \rangle$.

{\bf Flip-$\theta$ attack.} This attack simply flips the reward from $\langle \theta, a \rangle$ to $\langle -\theta, a \rangle$.  
This attack can be considered as highly aggressive, potentially using the budget quickly to the rewards appear to be the complete opposite of what they really are.

\section{Additional Experimental Details and Results} \label{sec:more_exp}

\subsection{Additional Details}

{\bf Details of contextual experiment.} The synthetic experimental setup of Section \ref{sec:exp_context} is detailed as follows.  We generate $k = 25$ ``center points'' $\mu_1,\dotsc,\mu_k$ (one per arm), each having entries drawn i.i.d.~from the uniform distribution on $\big[-\frac{1}{\sqrt d}, \frac{1}{\sqrt d}\big]$. The contexts $\{a_{i,t}\}_{i=1}^k$ at each time $t$ are then created by letting $a_{i,t} = \mu_i + \xi_{i,t}$, where $\xi_{i,t}$ are i.i.d.~$\mathcal{N}\big( 0,\frac{\eta^2}{d} I_d \big)$ for some variance $\eta^2 > 0$.  We fix the true parameter vector as $\theta = \big(\frac{1}{\sqrt d},\dotsc,\frac{1}{\sqrt d}\big)$,  
and we assume that observations are subject to $\mathcal{N}\big( 0, \frac{\sigma^2}{d} I_d \big)$ noise with $\sigma^2 = 0.05$.

{\bf Details of MovieLens experiment.} The MovieLens experimental setup of Section \ref{sec:exp_context} follows \cite{bogunovic2018adv}, and is detailed as follows.  The data for 1682 movies and 943 users takes the form of an incomplete matrix $\R$ of ratings, where $R_{i,j}$ is the rating of movie $i$ given by the user $j$. To impute the missing rating values, we apply non-negative matrix factorization with $d=15$ latent factors. This produces a feature vector for each movie $\m_i \in \mathbb{R}^{d}$ and user $\vu_j \in \mathbb{R}^{d}$. We use $10\%$ of the user data for training, in which we fit a Gaussian distribution $\mathcal{N}(\vu| \mu, \Sigma)$.  The reward for movie $i$ is given by $\langle \m_i, \vu_j \rangle$ for some fixed $j$.

{\bf Changes to the robust PE algorithm.} The parameters in Algorithm \ref{alg:cpe} (e.g., $\hat{C}_h$ and $m_0$) were chosen for convenience in the theoretical analysis that ignores constants, but we found that alternative choices are preferable in practice.  Accordingly, we run the algorithm with the following modifications: (i) $m_0 = d$; (ii) $\hat{C}_h = \min\{\sqrt{T}, 2^{\log_2 T - h}\}$; (iii) the right-hand side of \eqref{eq:retain_arms_condition} is replaced by $2\sqrt{\tfrac{4d}{m_{h}}\log\big(\tfrac{1}{\delta}\big)} + \tfrac{2\hat{C}_h}{m_h} \sqrt{4d}$; (iv) We fix $\delta = 0.1$ and $\nu = 0.05$.  Thus, the key changes are removing the $m_0$ terms from $\hat{C}_h$, and removing the division by $\nu$ in the elimination condition. 

\subsection{Additional Results}

{\bf Contextual setting: Budget vs.~Regret.} In Figure \ref{fig:attacks2}, we provide analogous plots to Figure \ref{fig:attacks} (Left) for all three algorithms (Greedy, LinUCB, and Thompson sampling) and two choices of $\eta$ ($0.2$ and $0.5$).  In all cases, we see a similar linear trend to that of Figure \ref{fig:attacks2}.

{\bf Non-contextual setting with 40 trials and unknown $C$.} In Figure \ref{fig:noncontext2}, we provide analogous plots to Figure \ref{fig:noncontext}, but with 40 trials instead of 10, and showing the worst 4 out of 40 curves instead of the worst 2 out of 10.  We observe similar findings to those discussed in Section \ref{sec:exp_noncontext}

{\bf Non-contextual setting with known $C$.} In Figure \ref{fig:noncontext_kc}, we provide analogous plots to Figure \ref{fig:noncontext} when Algorithm \ref{alg:cpe} is used with known $C$, with the following modifications similar to the unknown $C$ case: (i) $m_0 = d$; (ii) the right-hand side of \eqref{eq:retain_arms_condition} is replaced by $2\sqrt{\tfrac{4d}{m_{h}}\log\big(\tfrac{1}{\delta}\big)} + \tfrac{C}{m_h} \sqrt{d}$; (iii) We fix $\delta = 0.1$ and $\nu = 0.05$.  The attack is chosen to start during the same epoch as the unknown $C$ case.  From the regret plots, we observe broadly similar behavior to the unknown $C$ case, with the exception that the regret is considerably lower when there is no attack (i.e., $C = 0$).  This is to be expected, since in the known $C$ case, knowing that $C = 0$ means that one can confidently eliminate arms much faster.

%


\begin{figure}[t]
    \centering
    \includegraphics[width=0.325\textwidth]{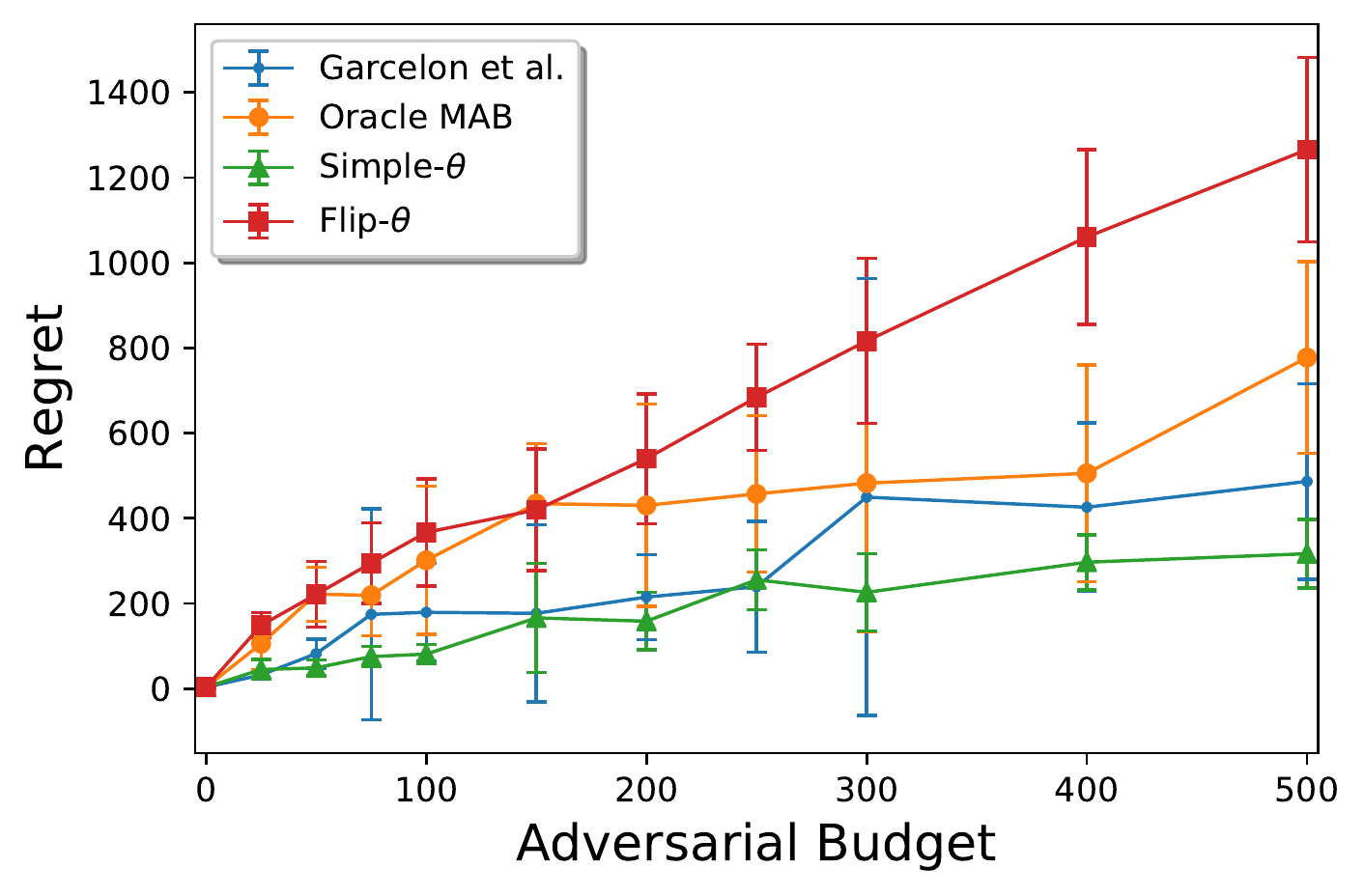}
    \includegraphics[width=0.325\textwidth]{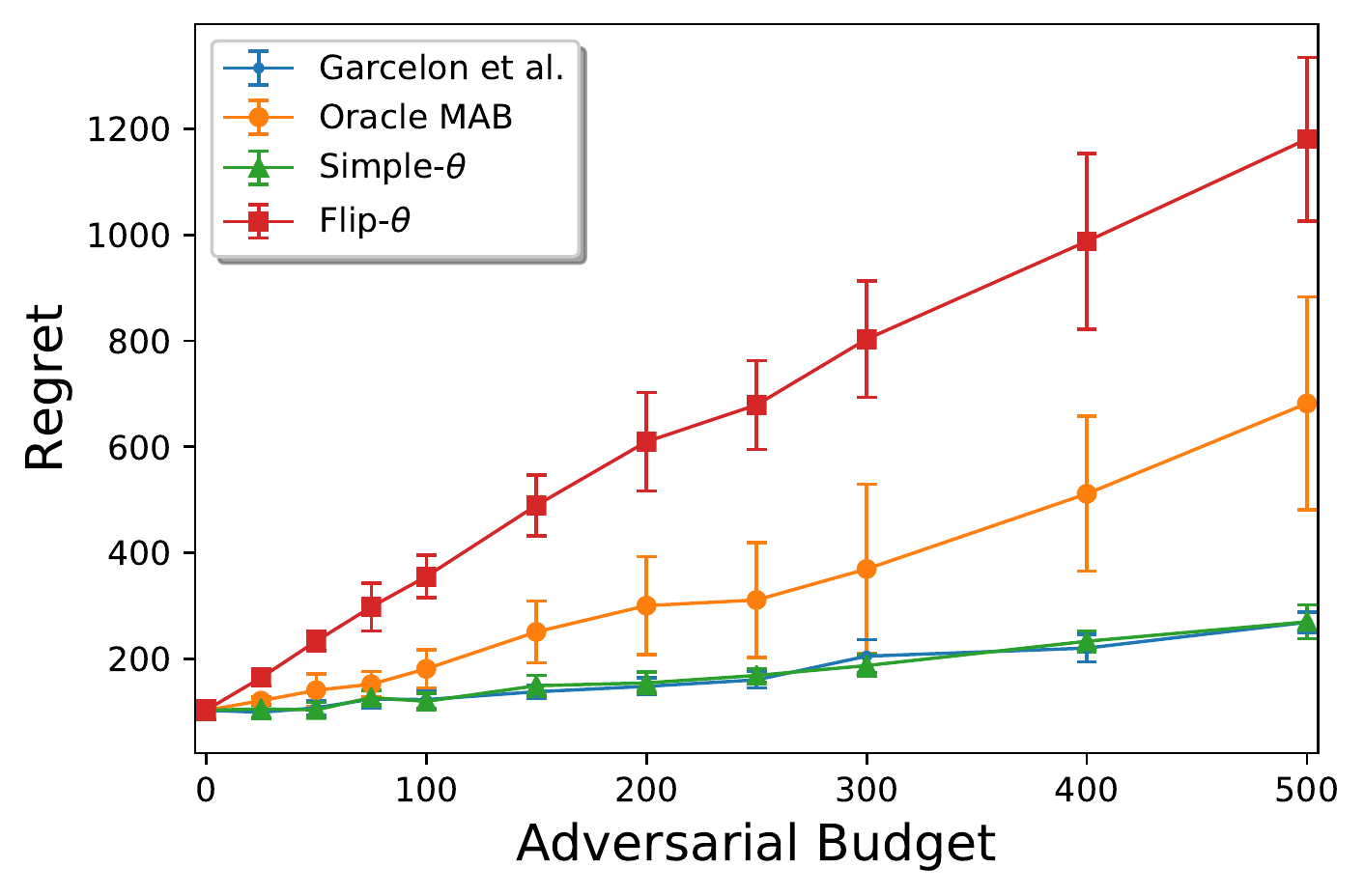}
    \includegraphics[width=0.325\textwidth]{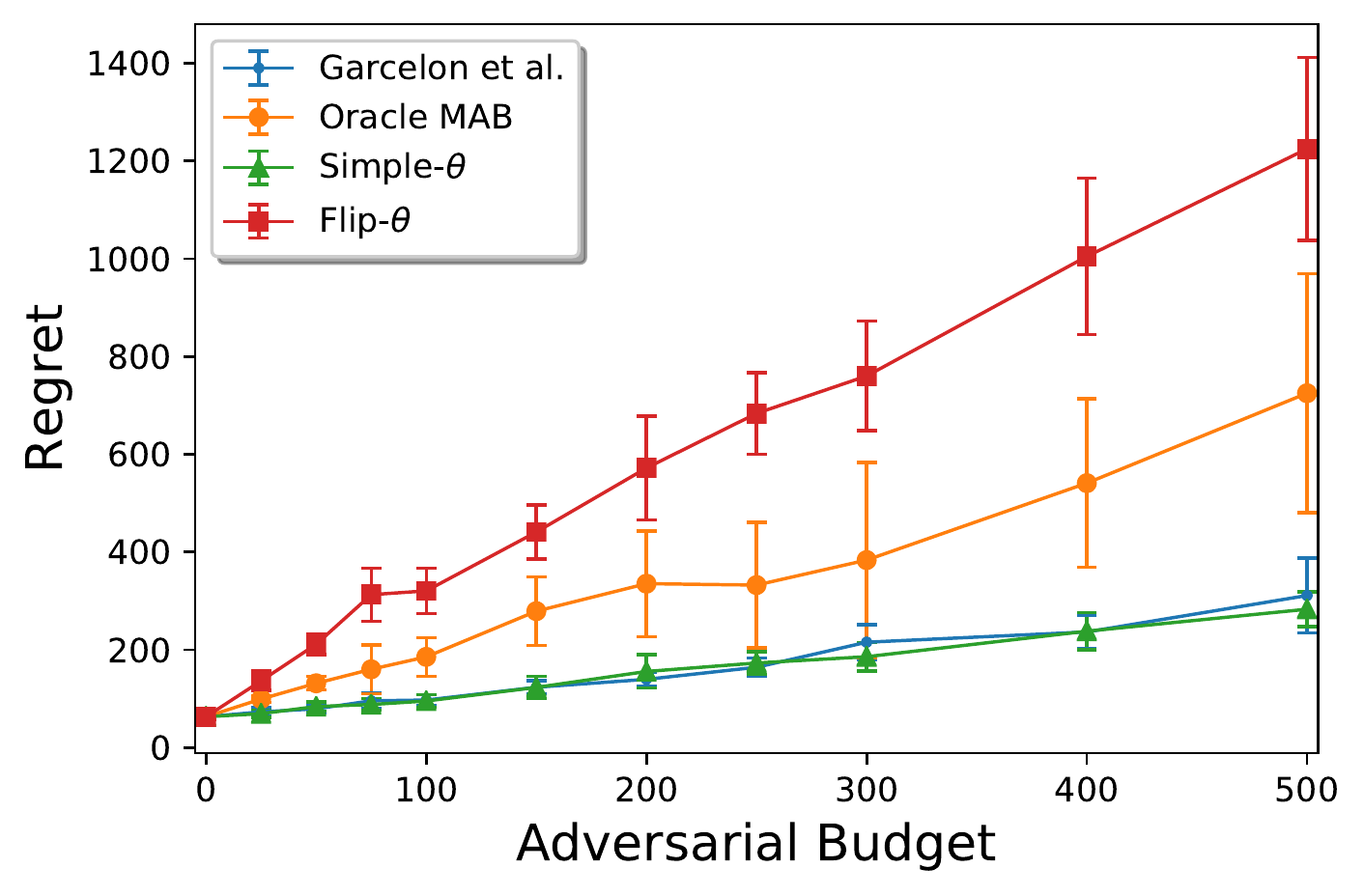}
    \includegraphics[width=0.325\textwidth]{figs/greedy_at_3500_sigma_05.pdf}
    \includegraphics[width=0.325\textwidth]{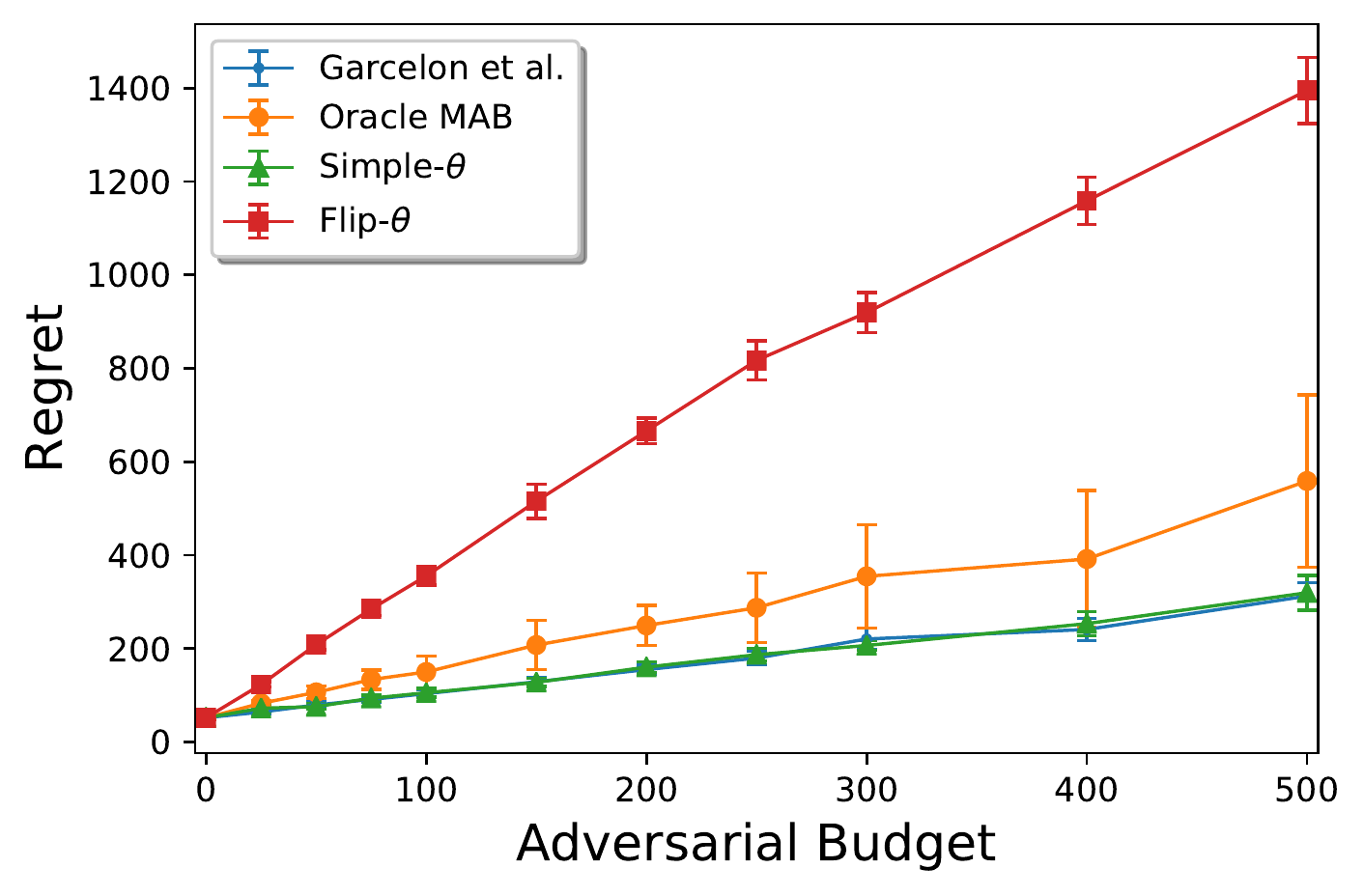}
    \includegraphics[width=0.325\textwidth]{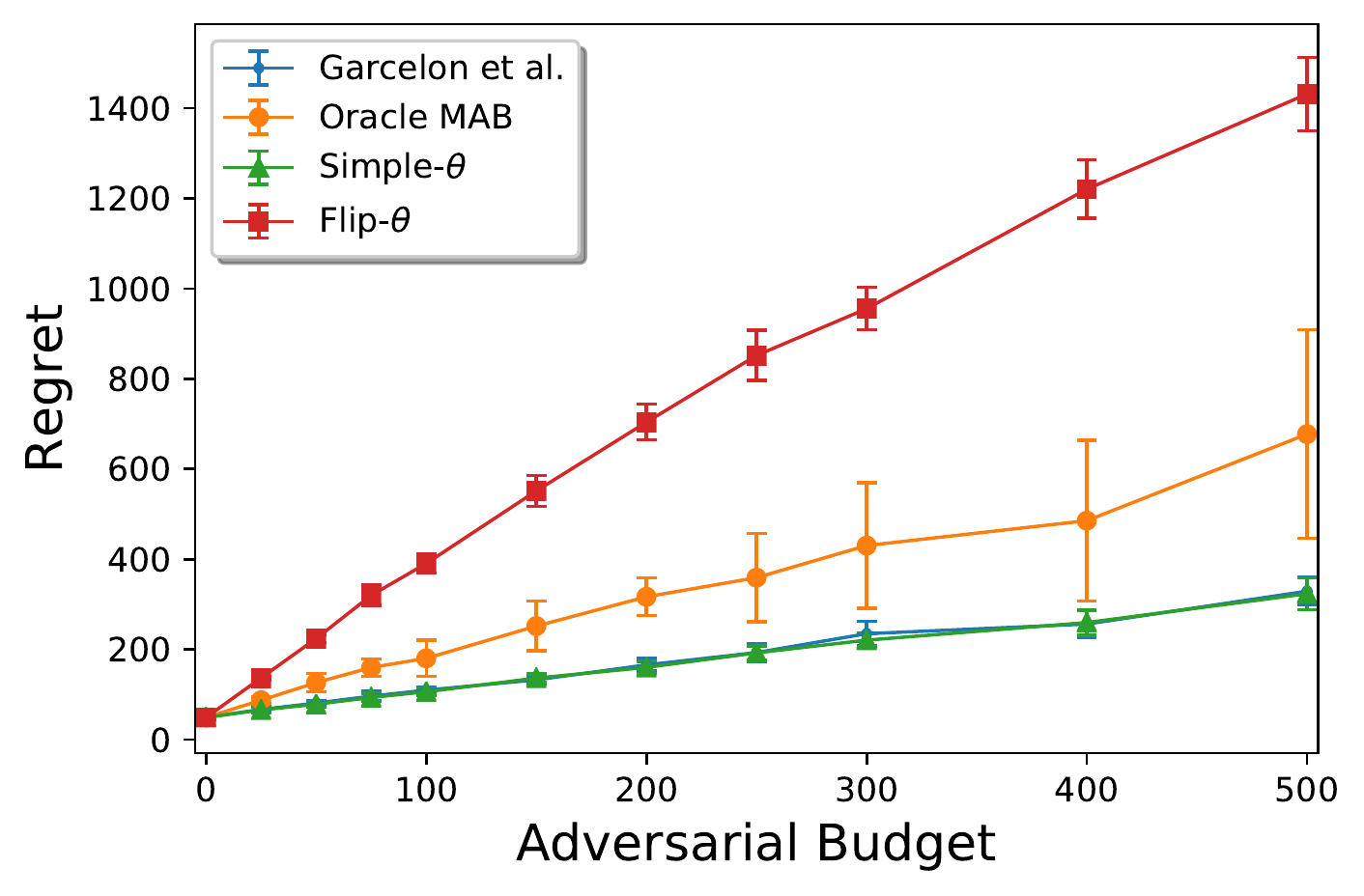}
    \caption{Contextual synthetic experiments: Regret as a function of $C$ with Greedy (Left), LinUCB (Middle), and Thompson Sampling (Right), under the perturbation levels $\eta = 0.2$ (Top) and $\eta = 0.5$ (Bottom).}
    \label{fig:attacks2}
    \vspace*{-1ex}  
\end{figure} 

\begin{figure}[t]
    \centering
    \includegraphics[width=0.325\textwidth]{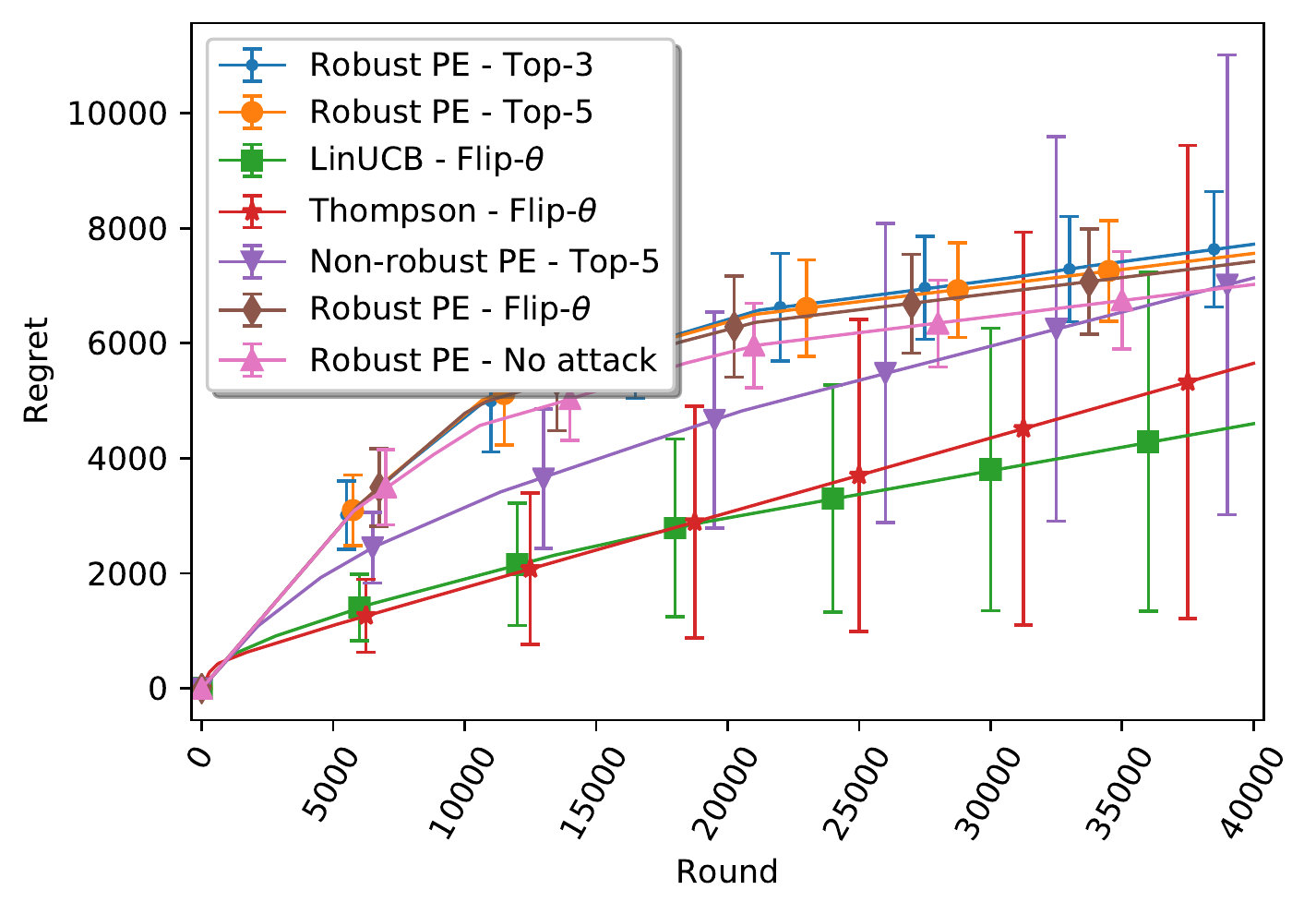}
    \includegraphics[width=0.325\textwidth]{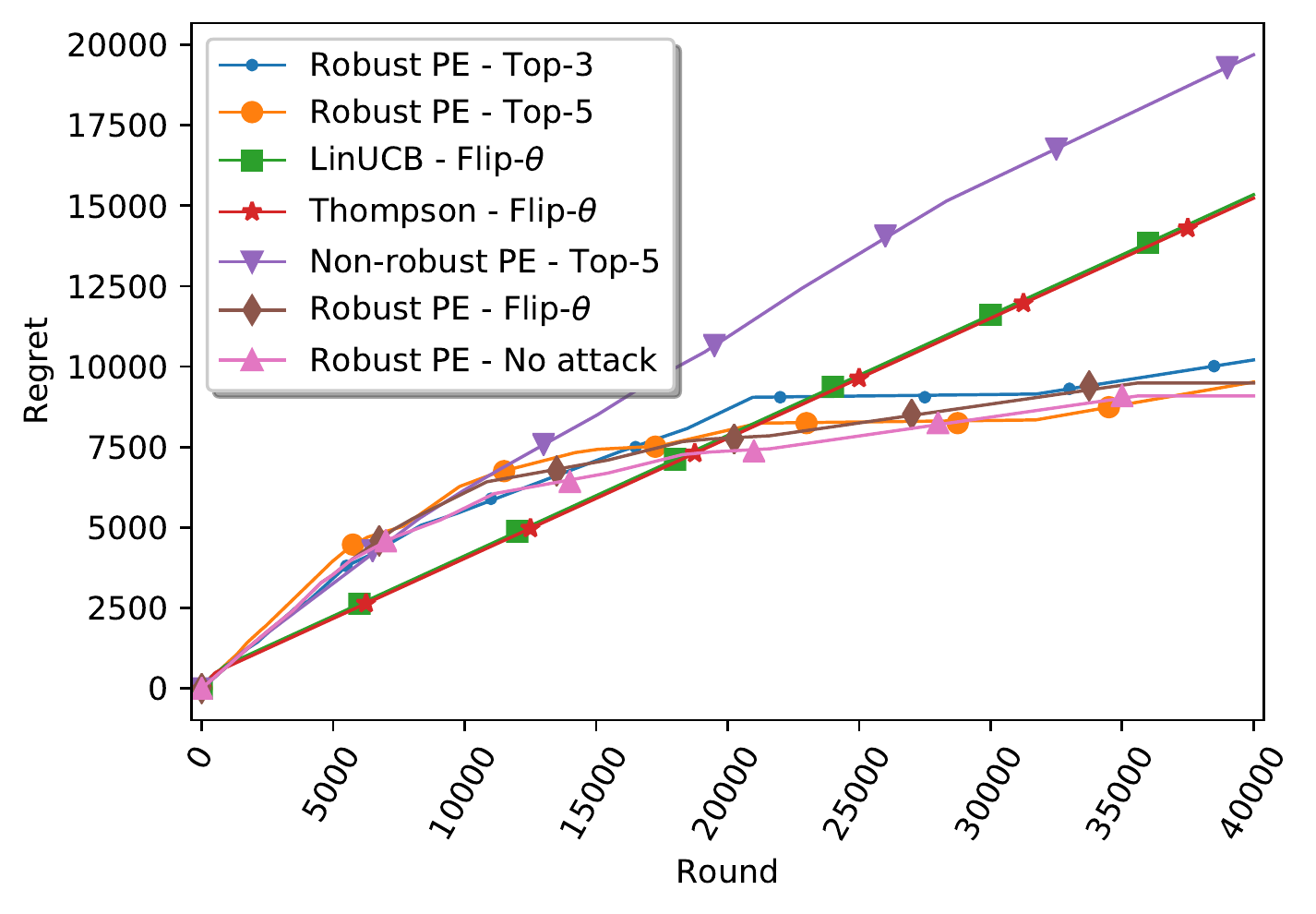}
    \includegraphics[width=0.325\textwidth]{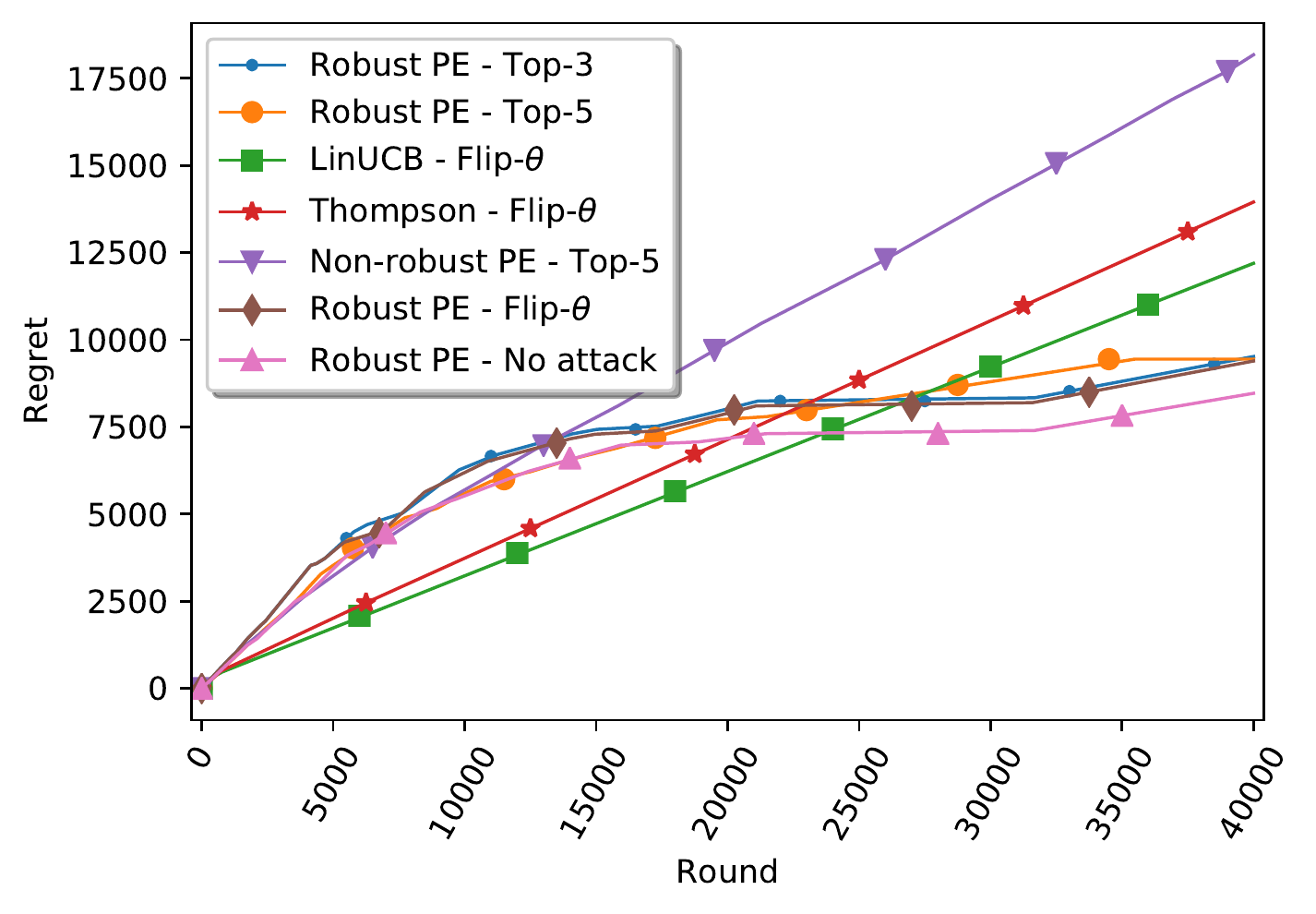}
    \includegraphics[width=0.325\textwidth]{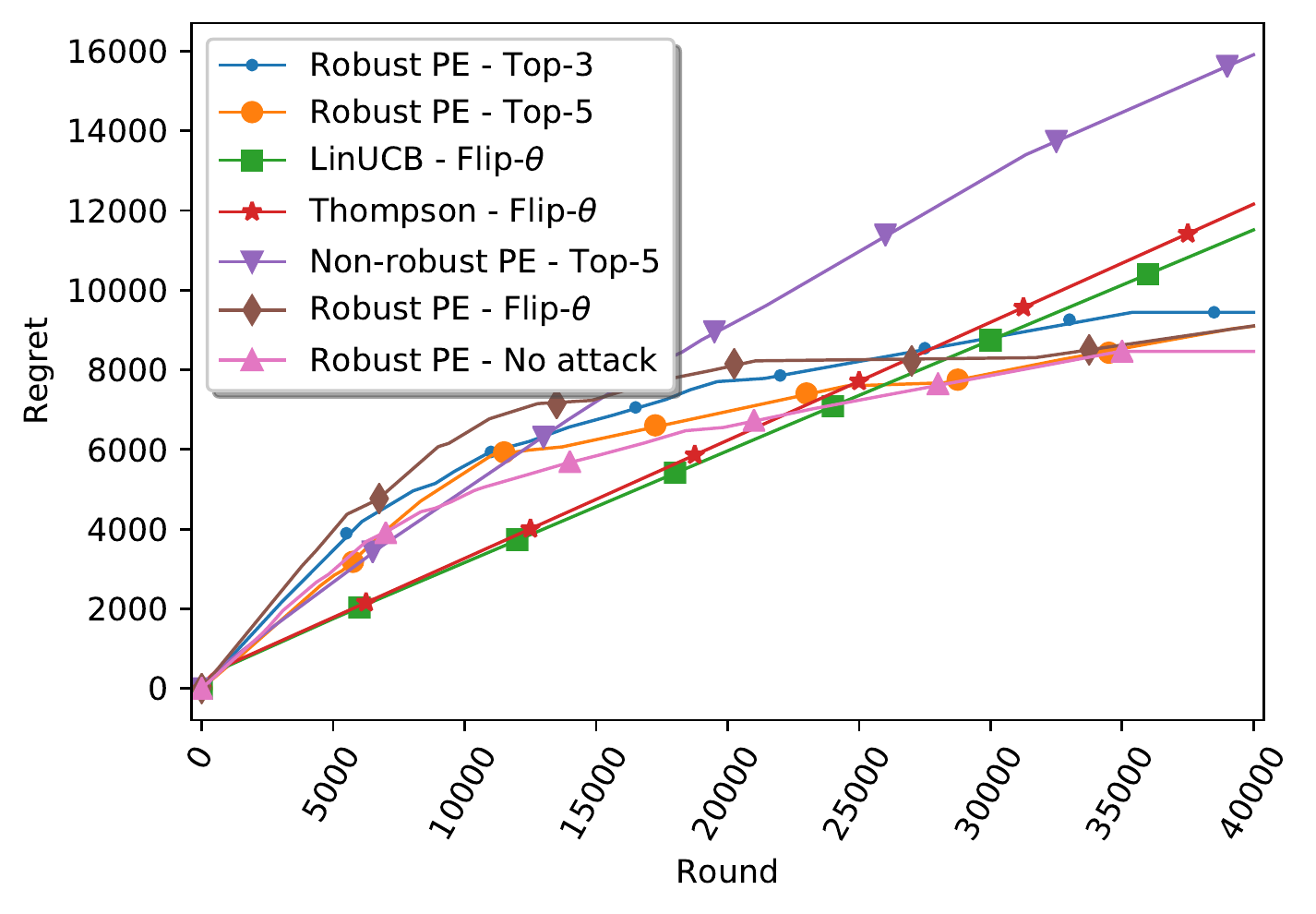}
    \includegraphics[width=0.325\textwidth]{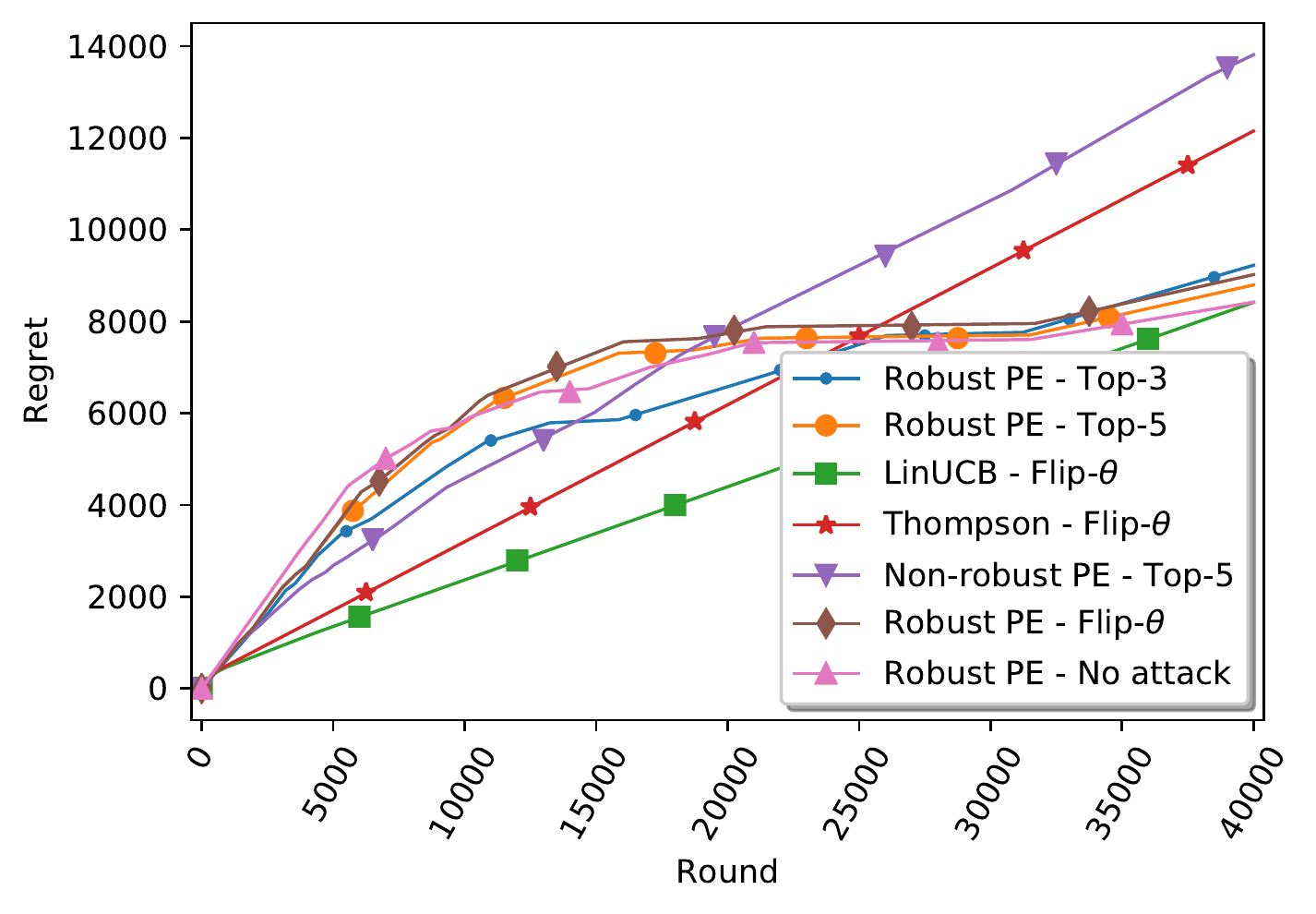}
    \caption{Non-contextual synthetic experiment with 40 trials: (Top-Left) Average regret as a function of time; (Remaining) Worst 4 runs among 40.}
    \label{fig:noncontext2}
    \vspace*{-1ex}  
\end{figure} 

\begin{figure}[t]
    \centering
    \includegraphics[width=0.3\textwidth]{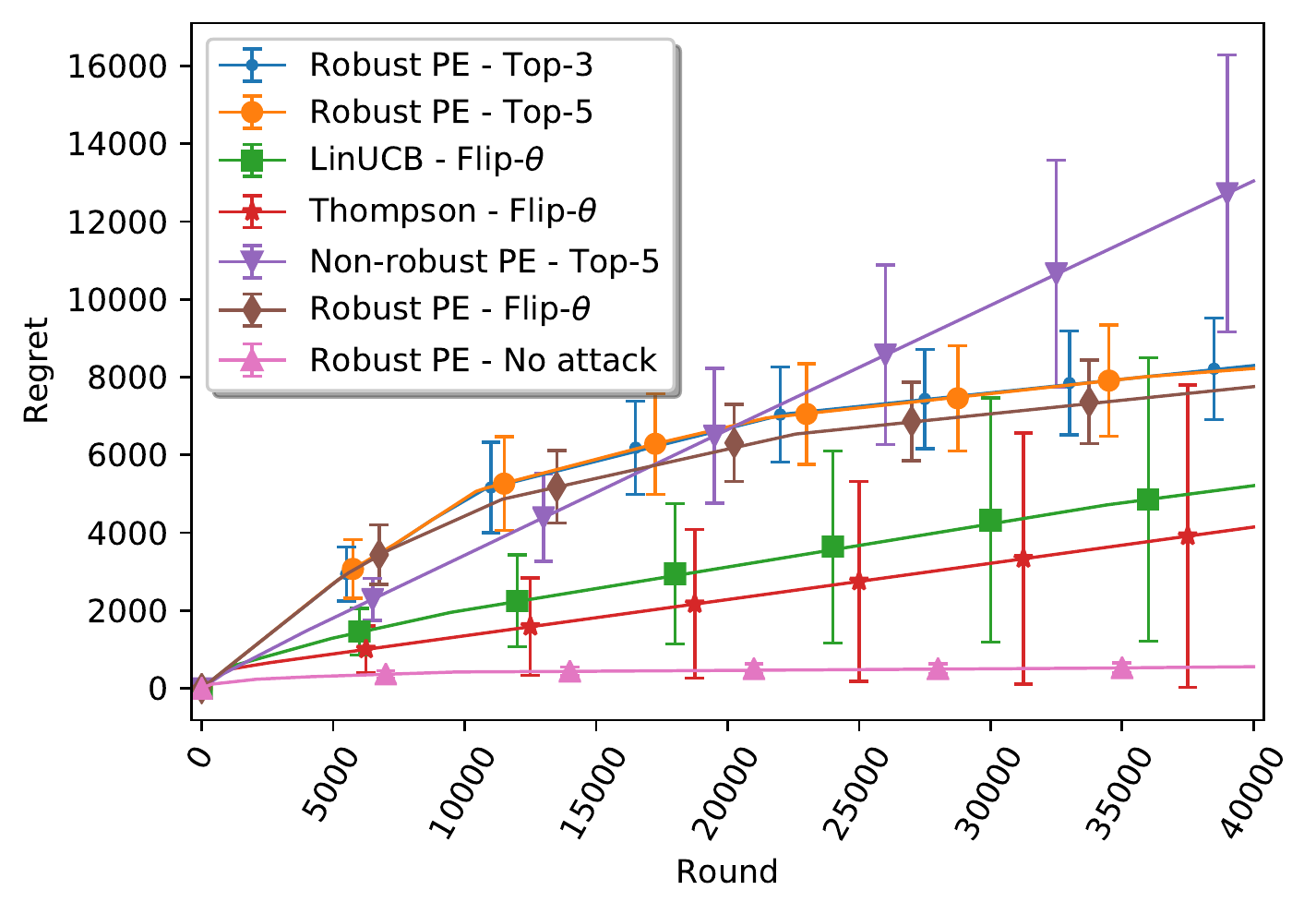}
    \includegraphics[width=0.3\textwidth]{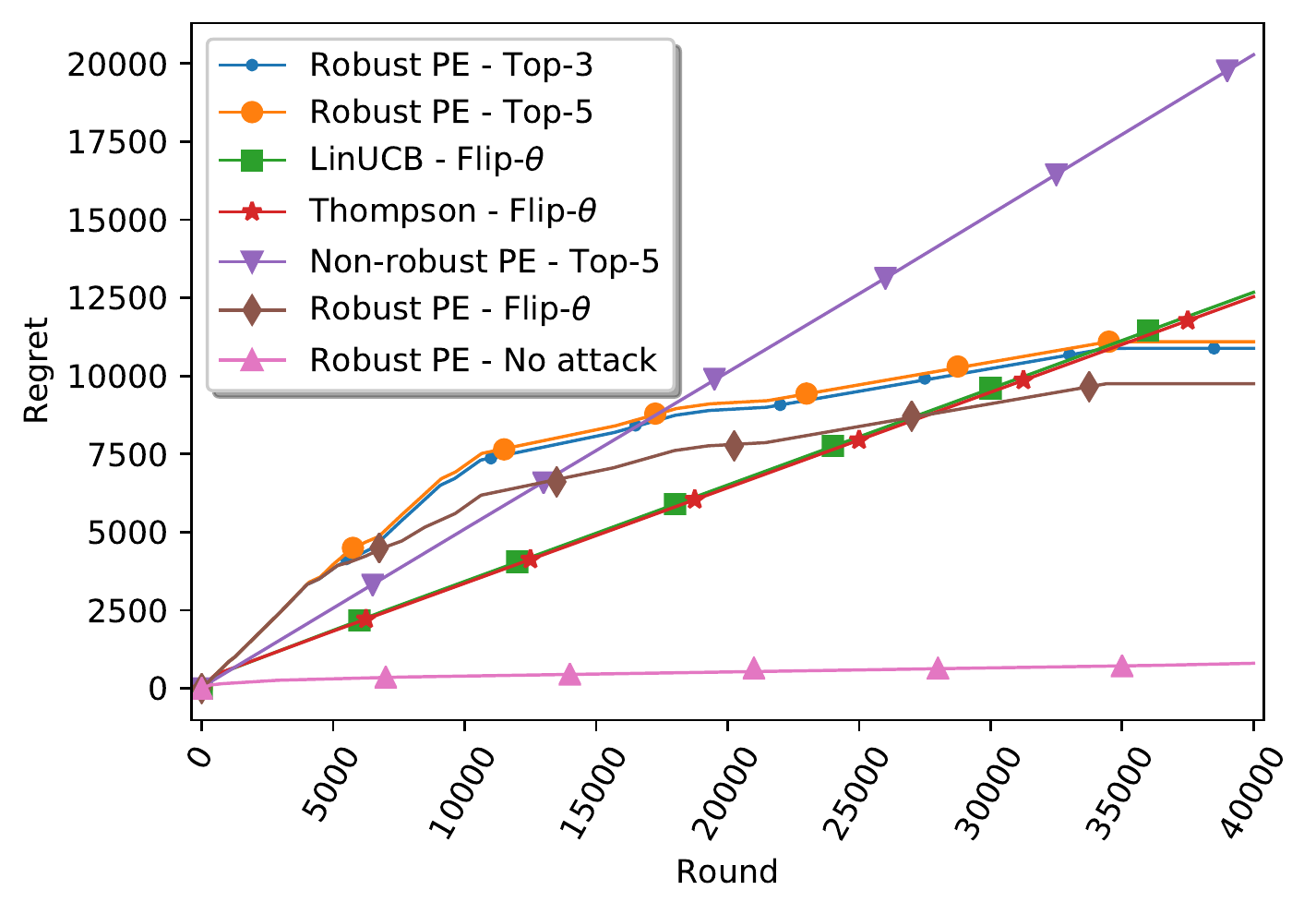}
    \includegraphics[width=0.3\textwidth]{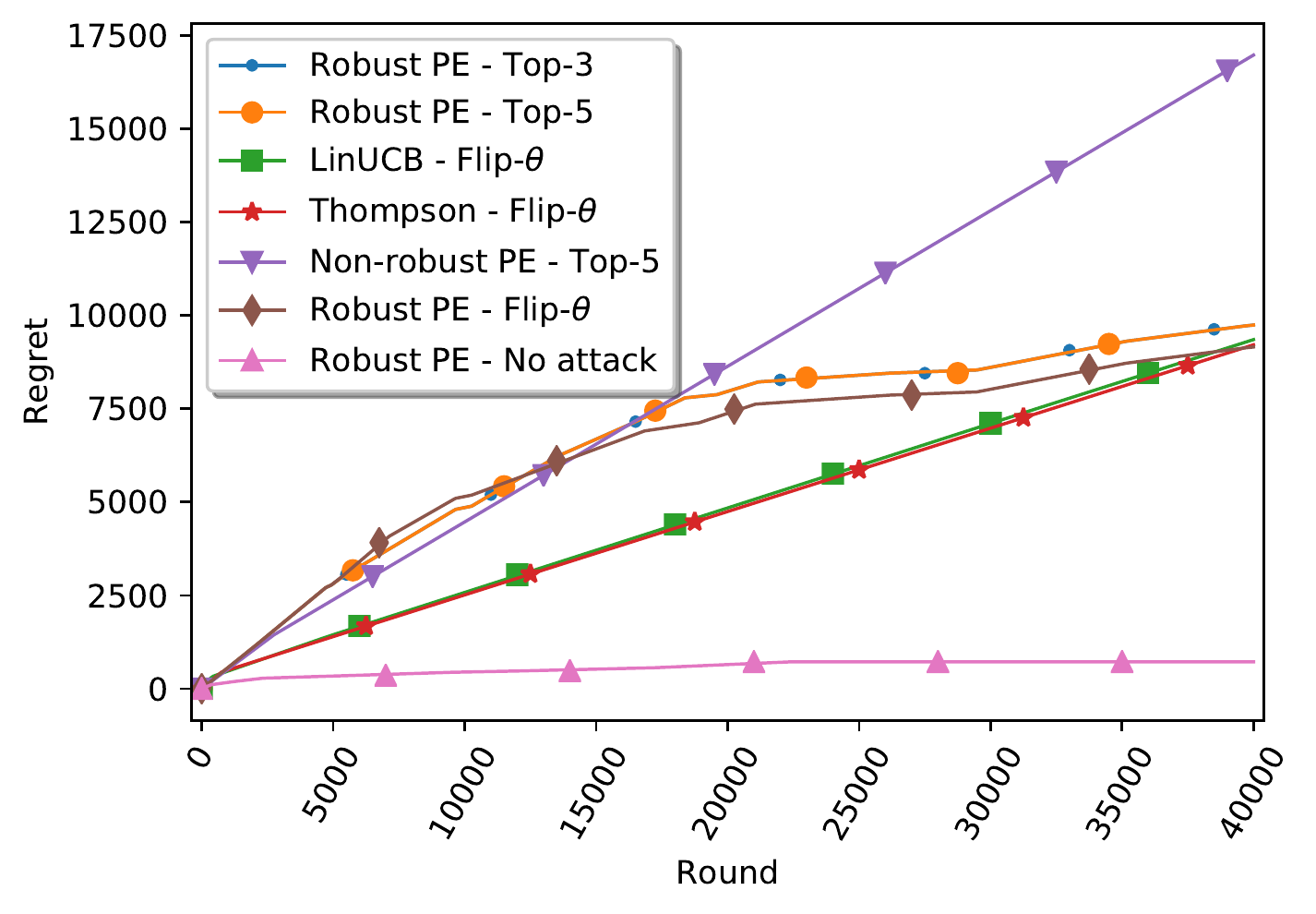}
    \caption{Non-contextual synthetic experiment with 10 trials and known $C$: (Left) Average regret as a function of time; (Middle) Worst run among 10; (Right) Second-worst run among 10.}
    \label{fig:noncontext_kc}
    \vspace*{-2.5ex}  
\end{figure} 

\section{Additional Related Works}\label{sec:more_rw}

\textbf{Adversarial bandits.} 
The most well-known adversarial multi-armed bandit problem is that in which the rewards are arbitrary, and the goal is to compete with the best fixed arm in hindsight \cite[Ch.~11]{lattimore2018bandit}, with the EXP3 algorithm being a particularly widely-adopted solution.  As discussed in \cite{lykouris2018stochastic}, in corrupted stochastic bandit problems, one still seeks to exploit the underlying stochastic structure, but in a more robust manner.  Accordingly, very different solutions are adopted.

A prominent adversarial setting specific to linear bandits is that handled by the EXP2 algorithm \cite{bubeck2012towards}.  However, this setting concerns {\em adversarial contexts} and {\em non-corrupted rewards}, whereas we consider{\em corrupted rewards}; as a result, the two settings and their bounds do not appear to be comparable.

\textbf{Best of both worlds.} The stochastic setting often leads to significantly smaller regret bounds, but at the expense of potentially restrictive modeling assumptions. Algorithms attaining the {\em best of both worlds} (stochastic and adversarial) \cite{bubeck2012best,seldin2014one,auer2016algorithm} are also related to the corruption-tolerant setting, but consider an unbounded adversary (and a different regret notion) in the adversarial case.  Hence, a key distinction is that the adversary's budget is ``all-or-nothing'' rather than being smoothly parametrized by $C$.  See \cite{lykouris2018stochastic} for further discussion.

\textbf{Model mismatch and misspecification.} A distinct but related direction in the linear bandit literature has been to address robustness to model mismatch and misspecification \cite{ghosh2017misspecified,krishnamurthy18semi} (see also \cite[Sec.~24.4]{lattimore2018bandit} and \cite{zanette2020learning}).  In \cite{ghosh2017misspecified}, the deviations for each arm are fixed and the same every round, whereas in \cite{krishnamurthy18semi} the deviations may depend on a context vector but not on the learner's action.  Hence, both can be viewed as considering a weaker adversary than the present paper.  On the other hand, this can lead to stronger regret guarantees, such as paying essentially no penalty under broad misspecification scenarios \cite{krishnamurthy18semi}.

\textbf{Fractional corruption model.} In \cite{kapoor_corruption-tolerant_2019}, a different corruption model was considered, both in the case of independent arms and linear rewards.  The constraint on the adversary therein is that at any time $t$, at most $\eta t$ fraction of the observed rewards have been corrupted by the adversary.  This is distinct from (and complementary to) the setting that we consider, in which the adversary can choose where to concentrate its budget.  The algorithms and bounds of \cite{kapoor_corruption-tolerant_2019} are not applicable in our setting.  

On the other hand, our algorithms and their regret bounds can be applied to the above-mentioned fractional corruption model, or even a more powerful corruption model in which the $\eta t$-fraction condition is only required to hold for $t=T$,\footnote{This variant is closely related to Huber's contamination model, in which each round may be adversarially corrupted independently with probability $\eta$.} provided that the adversary's corruption level is bounded by some value $B_0$ on each round that it corrupts (as is also assumed in \cite{kapoor_corruption-tolerant_2019}).  In particular, such behavior is already captured by our adversary upon setting $C = B_0 \eta T$.

\section{Discussion on Instance-Dependent vs.~Instance-Independent Bounds} \label{app:transfer}

In the setup of polyhedral domains considered in \cite{li2019stochastic}, under the assumption that the separation between the best and second-best corner point is at least $\Delta$, a regret bound of $O\big( \frac{C d^2 \log T}{\Delta} + \frac{d^5 \log \frac{d \log T}{\delta} \log T}{\Delta^2} \big)$ is proved, assuming a weaker adversary that knows the player's randomized action-selection distribution but not the specific action chosen.  To simplify the discussion, we assume $d = O(1)$ and $\Delta = \Theta(1)$, and ignore $\log T$ factors, thus focusing on the expression $\frac{C}{\Delta} + \frac{1}{\Delta^2}$ (which is slightly smaller than the above regret bound).

Consider a ``hard'' instance in which {\em every} suboptimal corner point has a gap of exactly $\Delta$ to the best corner point.  Then, the regret is trivially upper bounded by $\Delta T$, and combining this with the bound above would give a regret bound of $O\big( \max\big\{ \frac{C}{\Delta} + \frac{1}{\Delta^2}, \Delta T \big\} \big)$ or higher.  The idea in converting to an instance-independent regret bound is to maximize over $\Delta$, yielding $\max_{\Delta} \max\big\{ \frac{C}{\Delta} + \frac{1}{\Delta^2}, \Delta T \big\}$.

In Remark \ref{rem:comparison}, we claimed that an instance-independent regret bound deduced from \cite{li2019stochastic} would scale as $R_T = O( T^{2/3} + \sqrt{CT} )$ at best.  The two terms here follow by simply substituting two choices of $\Delta$ above: One that equates $\frac{C}{\Delta}$ with $\Delta T$ above (i.e., $\Delta = \sqrt{\frac{C}{T}}$), and one that equates $\frac{1}{\Delta^2}$ with $\Delta T$ (i.e., $\Delta = T^{-1/3}$).  The corresponding equated terms scale as $\sqrt{CT}$ and $T^{2/3}$ respectively, which establishes the desired claim.

\end{document}